\documentclass{article}
\pdfoutput=1
\usepackage{microtype}
\usepackage{graphicx}
\usepackage{booktabs}
\usepackage{bbm}
\usepackage{amsmath}
\usepackage{amssymb}
\usepackage{mathtools}
\usepackage{amsthm}
\usepackage{algorithm}   
\usepackage{algorithmic}
\usepackage{amsfonts}
\usepackage{multirow}
\usepackage{subcaption}

\usepackage{hyperref}
\usepackage[capitalize,noabbrev]{cleveref}
\usepackage[square,numbers,sort&compress]{natbib}
\usepackage[preprint]{icml2026}

\theoremstyle{plain}
\newtheorem{theorem}{Theorem}[section]

\newtheorem{corollary}[theorem]{Corollary}
\theoremstyle{definition}

\newtheorem{assumption}[theorem]{Assumption}
\theoremstyle{remark}
\newtheorem{remark}[theorem]{Remark}

\usepackage[T1]{fontenc}
\definecolor{myblue}{RGB}{70,130,180}
\definecolor{mygreen}{RGB}{60,179,113}
\definecolor{myorange}{RGB}{255,140,0}
\definecolor{mygray}{RGB}{100,100,100}
\definecolor{backpropcolor}{RGB}{220,20,60}

\usepackage[textsize=tiny]{todonotes}

\icmltitlerunning{From Simple to Complex: Curriculum-Guided Physics-Informed Neural Networks via Gaussian Mixture Models}

\begin{document}
	
	\twocolumn[
	\icmltitle{From Simple to Complex: Curriculum-Guided Physics-Informed Neural Networks via Gaussian Mixture Models}

	\begin{icmlauthorlist}
		\icmlauthor{Jianan Yang}{xjtu}
		\icmlauthor{Yiran Wang}{xjtu}
		\icmlauthor{Shuai Li}{xjtu}
		\icmlauthor{Fujun Cao}{zyu,shaoxinglab}
		\icmlauthor{Xuefei Yan}{ams}
		\icmlauthor{Junmin Liu}{xjtu}
	\end{icmlauthorlist}
	
	\icmlaffiliation{xjtu}{Xi'an Jiaotong University, Xi'an, China}
	\icmlaffiliation{zyu}{College of International Business, Zhejiang Yuexiu University, Shaoxing, China}
	\icmlaffiliation{shaoxinglab}{Shaoxing Key Laboratory of Intelligent Monitoring and Prevention of Smart Society, Shaoxing, China}
	\icmlaffiliation{ams}{Institute of Systems Engineering, AMS, Beijing, China}
	\icmlcorrespondingauthor{Junmin Liu}{junminliu@mail.xjtu.edu.cn}
	
	\icmlkeywords{Physics-Informed Neural Networks, Partial Differential Equations, Gaussian Mixture Models, Curriculum Learning}
	
	\vskip 0.3in
	]
	
	\printAffiliationsAndNotice{}
	
	\begin{abstract}
		Physics-informed neural networks (PINNs) offer a mesh-free framework for solving partial differential equations (PDEs), yet training often suffers from gradient pathologies, spectral bias, and poor convergence, especially for problems with strong nonlinearity, sharp gradients, or multiscale features. We propose the Curriculum-Guided Gaussian Mixture Physics-Informed Neural Network (CGMPINN), which integrates Gaussian mixture modeling with dynamic curriculum learning. Specifically, a GMM is periodically fitted to the PDE residual distribution to quantify spatially varying learning difficulty. A smooth curriculum schedule progressively shifts training focus from easy to harder regions, while precision-based variance modulation suppresses unreliable clusters during early optimization. This dual curriculum is governed by a shared curriculum parameter and can be combined with self-adaptive loss balancing. We further establish theoretical guarantees, including sublinear convergence of the gradient norm for the induced time-varying loss, uniform equivalence between the curriculum-weighted and standard PDE losses, and a generalization bound with an explicit weighting-induced bias characterization. Experiments on six benchmark PDEs spanning elliptic, parabolic, hyperbolic, advection-dominated, and nonlinear reaction-diffusion types show that CGMPINN consistently achieves the lowest relative $L_2$ and maximum absolute errors among all compared methods, reducing relative $L_2$ error by up to 97.8\% over the standard PINN at comparable cost. Our code is publicly available at \url{https://github.com/Mathematics-Yang/CGMPINN}.
	\end{abstract}
	
	\section{Introduction}
	
	Partial differential equations (PDEs) are the foundational mathematical language for modeling complex physical phenomena across scientific and engineering disciplines \cite{PINN_Raissi1,Karniadakis2}. Classical numerical methods such as the finite difference method (FDM), finite element method (FEM), and spectral methods \cite{Zienkiewicz3,Fornberg4} have achieved remarkable success, yet encounter inherent difficulties in high-dimensional problems, multi-scale systems, complex geometries, and inverse problems with sparse and noisy data \cite{Han5,Sirignano6}.
	
	Motivated by these challenges, physics-informed neural networks (PINNs), pioneered by Raissi et al. \cite{PINN_Raissi1}, have emerged as a promising mesh-free paradigm for PDE solving. PINNs encode governing PDEs and initial/boundary conditions (I/BCs) as soft regularization terms in the loss function of a neural network \cite{Lagaris7,Dissanayake8}. By minimizing a composite loss combining data misfit and PDE residual, PINNs learn continuous, mesh-free solution approximations without requiring large labeled datasets \cite{Yang9,Zhang10}, enabling diverse applications in fluid mechanics \cite{Raissi11}, biomedical engineering \cite{Sahli12}, geophysical inversion \cite{Ishitsuka13}, and nano-optics \cite{Chen14}, as reviewed in \cite{Karniadakis2,Cuomo15}.
	
	However, vanilla PINNs exhibit well-documented failure modes. Imbalanced gradients between the PDE residual loss and the I/BCs loss can lead to stiff gradient flow and convergence to suboptimal minima \cite{Wang16,Wang17}, particularly for PDEs with strong nonlinearity, convection dominance, or multi-scale features \cite{Krishnapriyan18,Fuks19}. In addition, the non-convex loss landscape, poor conditioning of the optimization problem, spectral bias of neural networks \cite{Rahaman21,Xu22}, and static uniform sampling of collocation points \cite{Daw23,Wu24} further degrade solution accuracy.
	
	Extensive research has been devoted to addressing these issues. Wang et al. \cite{Wang16} proposed gradient-statistics-based learning rate annealing for adaptive loss balancing, followed by other adaptive loss weighting \cite{Zhu25,Xiang26} and gradient normalization techniques \cite{Chen27}. Yu et al. \cite{Yu28} introduced gPINN, a gradient-enhanced formulation incorporating PDE derivative information into training. Tao et al. \cite{Tao29} developed LNN-PINN, a liquid residual gating architecture, to improve predictive accuracy. Dodge et al. \cite{Dodge30} developed STAR-PINN, a stacked adaptive residual architecture, to enhance convergence stability and computational efficiency. From the sampling perspective, Wu et al. \cite{Wu24} presented a comprehensive study of sampling strategies for physics-informed neural networks, covering both non-adaptive and residual-based adaptive schemes, while Daw et al. \cite{Daw23} proposed the retain-resample-release (R3) framework. Importance sampling \cite{Nabian31,Zhang32} and Gaussian mixture distribution-based methods \cite{Jiao33} have also been explored for targeted collocation point placement. Additionally, domain decomposition approaches \cite{Shukla34,Jagtap35} and causal training strategies \cite{Wang36} address optimization complexity and temporal causality, respectively, while Bayesian extensions \cite{Zhang37} and adversarial learning-based methods \cite{Yang38} have been explored for uncertainty quantification.
	
	Among these strategies, curriculum learning (CL) \cite{Bengio39} has emerged as a particularly effective paradigm that progressively increases problem difficulty during training. While Krishnapriyan et al. \cite{Krishnapriyan18} confirmed the lack of universal effectiveness of fixed-step curriculum learning across multiple benchmarks, recent studies have demonstrated that dynamic curriculum regularization, integrated with early stopping mechanisms, can significantly enhance PINN training \cite{Monaco40,Duffy41}. More recently, curriculum-enhanced adaptive sampling \cite{Cetinkaya42} and adaptive task decomposition strategies \cite{Yang43} have integrated progressive scheduling with residual-based refinement. However, most existing curriculum methods rely on manually predefined, fixed-step schedules requiring extensive problem-specific tuning \cite{Krishnapriyan18,Monaco40}, and even dynamic schemes \cite{Duffy41} primarily adjust global PDE parameters without quantifying the spatially varying difficulty. To effectively capture such localized and multi-modal error landscapes, Gaussian mixture models (GMMs) provide a principled way to identify regions with heterogeneous residual structure, uncertainty, and learning difficulty. Although GMMs have been successfully employed for adaptive sampling \cite{Daw23,Wu24,Nabian31,Zhang32,Jiao33}, their use as an explicit mechanism for curriculum learning in PINNs remains limited. This motivates a data-driven curriculum strategy in which probabilistic information extracted from the residual distribution is used to organize training from easier regions to more difficult ones.
	
	Based on this insight, we propose the Curriculum-Guided Gaussian Mixture Physics-Informed Neural Network (CGMPINN), a unified framework that integrates Gaussian mixture modeling with dynamic curriculum learning for PINN training. The main contributions of this work are summarized as follows:
	\begin{itemize}
		\item We propose CGMPINN, a unified framework that integrates Gaussian mixture modeling with dynamic curriculum learning for PINN training. A GMM is periodically fitted to the PDE residual distribution to quantify spatially varying learning difficulty, and a smooth curriculum schedule progressively shifts training focus from easy to hard regions. A precision-based variance modulation mechanism further suppresses unreliable clusters during early optimization. The entire curriculum is governed by a single shared parameter and can be seamlessly combined with self-adaptive loss balancing.
		\item We establish three formal theoretical guarantees for CGMPINN: (i)~uniform equivalence between the curriculum-weighted and standard PDE losses; (ii)~sublinear convergence of the gradient norm for the induced time-varying total loss under summable objective drift; and (iii)~a Rademacher-complexity-based generalization bound for the weighted empirical PDE loss, together with an explicit characterization of the weighting-induced bias.
		\item We conduct systematic experiments on six benchmark PDEs spanning elliptic (1D and 2D), parabolic, hyperbolic, advection-dominated, and nonlinear reaction-diffusion regimes. Under identical architectures and training settings, CGMPINN consistently outperforms the standard PINN \cite{PINN_Raissi1}, lbPINN \cite{Xiang26}, gPINN \cite{Yu28}, LNN-PINN \cite{Tao29}, and STAR-PINN \cite{Dodge30} across all error metrics, reducing the relative $L_2$ error by up to 97.8\% at comparable computational cost. An ablation study further confirms the complementary roles of the GMM-based difficulty quantification and the curriculum scheduling components.
	\end{itemize}
	
	The remainder of this paper is organized as follows. Section~2 presents the proposed CGMPINN framework, including the curriculum-guided Gaussian mixture module, self-adaptive loss balancing, and the theoretical analysis of convergence, loss equivalence, and generalization. Section~3 details the numerical experiments and comparative results. Section~4 concludes the paper and discusses future directions.
	
	\section{CGMPINN Framework}
	\label{sec:cgmpinn}
	
	In this section, we present the CGMPINN framework. Following the problem formulation and a brief review of standard PINNs, we detail the two core components of our proposed approach: a \emph{curriculum-guided Gaussian mixture} (CGM) module for difficulty-aware residual weighting, and an optional Relative Loss Balancing with Random Lookback (ReLoBRaLo)-based self-adaptive loss balancing mechanism. 
	The CGM module jointly exploits GMM posterior responsibilities,
	component-level precision, and a smooth easy-to-hard curriculum
	schedule.
	The end-to-end architecture is illustrated in
	Figure~\ref{fig:cgmpinn_architecture}.
	
	\begin{figure}[htbp]
		\centering
		\includegraphics[width=0.5\textwidth]{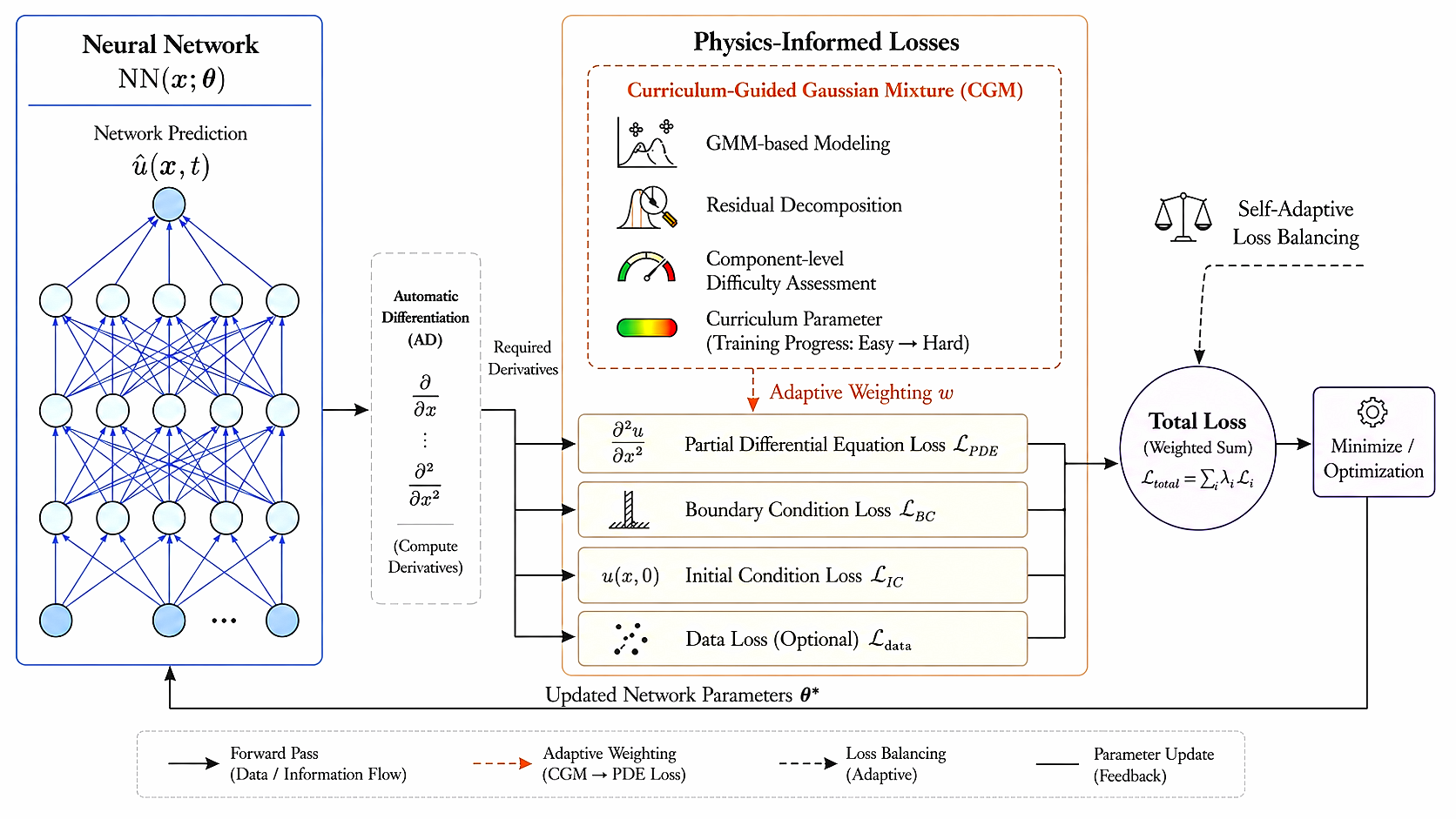}
		\caption{Architecture of the CGMPINN framework:
			a neural network approximator, automatic differentiation
			for PDE-compliant derivatives, the CGM module
			(GMM fitting, precision modulation, and curriculum scheduling),
			and an optional self-adaptive loss balancing mechanism.}
		\label{fig:cgmpinn_architecture}
	\end{figure}
	
	\subsection{Problem Formulation}
	
	We consider a general initial-boundary value problem on a bounded
	Lipschitz domain $\Omega\subset\mathbb{R}^d$ ($d$ is the spatial dimension) over $[0,T]$:
	\begin{equation}
		\begin{cases}
			\mathcal{D}[u](\boldsymbol{x},t)=f(\boldsymbol{x},t),
			& (\boldsymbol{x},t)\in\Omega\times[0,T],\\[2pt]
			\mathcal{B}[u](\boldsymbol{x},t)=g(\boldsymbol{x},t),
			& (\boldsymbol{x},t)\in\partial\Omega\times[0,T],\\[2pt]
			u(\boldsymbol{x},0)=u_0(\boldsymbol{x}),
			& \boldsymbol{x}\in\Omega,
		\end{cases}
		\label{eq:pde_ibvp}
	\end{equation}
	where $\boldsymbol{x}\in\Omega$ denotes the spatial coordinate,
	$t\in[0,T]$ is the time variable with final time $T>0$,
	$u:\Omega\times[0,T]\to\mathbb{R}$ is the unknown scalar field,
	$\mathcal{D}$ is a possibly nonlinear differential operator,
	$\mathcal{B}$ encodes boundary conditions
	(Dirichlet, Neumann, or Robin),
	$\partial\Omega$ denotes the boundary of $\Omega$,
	$f$ is the source term, $g$ is the boundary data,
	and $u_0$ is the initial condition.
	For stationary problems the temporal dependence and initial condition
	are omitted.
	
	\subsection{Standard PINN Formulation and Limitations}
	
	In the PINN framework~\cite{PINN_Raissi1}, the solution is approximated by
	$\hat{u}(\boldsymbol{x},t)=\mathrm{NN}(\boldsymbol{x},t;\,\theta)$,
	where $\mathrm{NN}(\cdot;\theta)$ denotes a neural network
	(e.g., a fully-connected feedforward network in our experiments)
	with trainable parameters $\theta\in\mathbb{R}^p$.
	Automatic differentiation \cite{Baydin44} provides the fundamental mechanism for computing exact derivatives in the residual-based optimization of physics-informed neural networks.
	Given $N_\Omega$ interior collocation points
	$\{(\boldsymbol{x}_i,t_i)\}_{i=1}^{N_\Omega}$,
	$N_{\partial\Omega}$ boundary points
	$\{(\boldsymbol{x}_j,t_j)\}_{j=1}^{N_{\partial\Omega}}$,
	and $N_{\mathrm{IC}}$ initial-condition points
	$\{\boldsymbol{x}_l\}_{l=1}^{N_{\mathrm{IC}}}$,
	the pointwise residuals are defined as
	\begin{align}
		r_{\mathrm{PDE}}(\boldsymbol{x}_i,t_i;\,\theta)
		&= \mathcal{D}[\hat{u}](\boldsymbol{x}_i,t_i)
		- f(\boldsymbol{x}_i,t_i),
		\label{eq:res_pde}\\
		r_{\mathrm{BC}}(\boldsymbol{x}_j,t_j;\,\theta)
		&= \mathcal{B}[\hat{u}](\boldsymbol{x}_j,t_j)
		- g(\boldsymbol{x}_j,t_j),
		\label{eq:res_bc}\\
		r_{\mathrm{IC}}(\boldsymbol{x}_l;\,\theta)
		&= \hat{u}(\boldsymbol{x}_l,0) - u_0(\boldsymbol{x}_l).
		\label{eq:res_ic}
	\end{align}
	For notational brevity, we write
	$r_{\mathrm{PDE},i} = r_{\mathrm{PDE}}(\boldsymbol{x}_i,t_i;\,\theta)$,
	and similarly $r_{\mathrm{BC},j}$ and $r_{\mathrm{IC},l}$ for the
	boundary and initial-condition residuals at their respective points.
	The standard training objective is the composite MSE loss
	\begin{equation}
		\begin{split}
			\mathcal{L}_{\mathrm{total}}(\theta)
			=\;&\frac{1}{N_\Omega}\sum_{i=1}^{N_\Omega}r_{\mathrm{PDE},i}^{2}
			+\lambda_{\mathrm{BC}}
			\frac{1}{N_{\partial\Omega}}
			\sum_{j=1}^{N_{\partial\Omega}}r_{\mathrm{BC},j}^{2}\\
			&+\lambda_{\mathrm{IC}}
			\frac{1}{N_{\mathrm{IC}}}
			\sum_{l=1}^{N_{\mathrm{IC}}}r_{\mathrm{IC},l}^{2},
		\end{split}
		\label{eq:standard_pinn_loss}
	\end{equation}
	where $\lambda_{\mathrm{BC}},\lambda_{\mathrm{IC}}$ are fixed penalty
	weights. Despite its elegance, this formulation has well-documented
	limitations: fixed weights cannot adapt to evolving gradient magnitudes,
	leading to gradient pathologies~\cite{Wang16}; moreover, all collocation
	points are treated uniformly regardless of learning difficulty, causing
	the optimizer to be overwhelmed by hard residuals early in training and
	often converging to suboptimal local minima~\cite{Krishnapriyan18}.
	CGMPINN addresses both challenges through the mechanisms detailed below.
	
	\subsection{Curriculum-Guided Gaussian Mixture Module}
	
	The CGM module is the core of CGMPINN. It fits a Gaussian mixture model
	(GMM) to the PDE residuals, extracting three complementary information
	channels: (i)~\emph{posterior responsibilities} (latent variables) that
	softly assign each collocation point to mixture components;
	(ii)~\emph{component-level difficulty} derived from
	responsibility-weighted residual magnitudes;
	and (iii)~\emph{component-level precision} (inverse variance)
	that quantifies the reliability of each cluster.
	A smooth curriculum schedule then combines these signals to
	transition training progressively from easy, well-conditioned residual
	regions to hard, uncertain ones.
	
	\subsubsection{Residual Distribution Modeling via GMM}
	
	At training iteration $k$, the PDE residuals at all $N_\Omega$ interior
	collocation points, evaluated at the network parameters
	$\theta_{k-1}$ obtained from the previous iteration, are collected as
	\begin{equation}
		\mathcal{R}^{(k)}
		= \bigl\{r_i \triangleq
		r_{\mathrm{PDE}}(\boldsymbol{x}_i,t_i;\,\theta_{k-1})
		\bigr\}_{i=1}^{N_\Omega}.
		\label{eq:residual_set}
	\end{equation}
	Their distribution is modeled by a $K$-component Gaussian mixture
	\begin{equation}
		p(r) = \sum_{m=1}^{K}\pi_m\,
		\mathcal{N}(r\mid\mu_m,\sigma_m^2),
		\label{eq:gmm_pdf}
	\end{equation}
	where $\pi_m\!\geqslant\!0$, $\sum_m\pi_m\!=\!1$ are mixing coefficients and
	$(\mu_m,\sigma_m^2)$ are the mean and variance of the $m$-th component.
	The parameters
	$\Theta_{\mathrm{GMM}}=\{\pi_m,\mu_m,\sigma_m^2\}_{m=1}^{K}$
	are estimated via the Expectation-Maximization (EM) algorithm.
	Upon convergence, the posterior responsibility of component~$m$
	for sample $r_i$ is
	\begin{equation}
		\gamma_{im}
		= \frac{\pi_m\,\mathcal{N}(r_i\mid\mu_m,\sigma_m^2)}
		{\sum_{l=1}^{K}\pi_l\,
			\mathcal{N}(r_i\mid\mu_l,\sigma_l^2)}.
		\label{eq:gmm_posterior}
	\end{equation}
	These posterior responsibilities serve as \emph{soft latent assignments},
	indicating the degree to which each collocation point belongs to each
	mixture component, and form the bridge connecting component-level
	information to sample-level weights.
	
	\subsubsection{Difficulty Quantification, Precision Modulation,
		and Curriculum Scheduling}
	
	\textbf{Component-level difficulty.}\;
	The difficulty of the $m$-th Gaussian component is quantified as the
	posterior-weighted mean squared residual:
	\begin{equation}
		d_m = \frac{\sum_{i=1}^{N_\Omega}\gamma_{im}\,r_i^2}
		{\sum_{i=1}^{N_\Omega}\gamma_{im}+\epsilon},
		\label{eq:component_difficulty}
	\end{equation}
	where $\epsilon>0$ is a small positive constant introduced
	throughout this section to prevent division by zero in normalization
	and weighting steps.
	Components with large $d_m$ correspond to regions where the current
	network produces large PDE residuals and are thus deemed more difficult.
	To enable comparable scaling, the scores are normalized to the unit
	interval:
	\begin{equation}
		\tilde{d}_m
		= \frac{d_m - d_{\min}}
		{d_{\max} - d_{\min} + \epsilon},
		\label{eq:difficulty_normalized}
	\end{equation}
	where $d_{\min}=\min_m d_m$ and $d_{\max}=\max_m d_m$.
	
	\textbf{Curriculum parameter.}\;
	A curriculum parameter $\tau(k)\in[0,1]$ increases monotonically with
	training iteration $k$, governing the smooth transition from easy-sample
	focus to hard-sample focus:
	\begin{equation}
		\tau(k)
		= \min\!\Bigl(\frac{k}{k_{\max}\cdot c_{\mathrm{sat}}},\;1\Bigr),
		\label{eq:tau}
	\end{equation}
	where $k_{\max}$ is the total number of iterations and
	$c_{\mathrm{sat}}\in(0,1]$ is the saturation fraction that
	specifies the proportion of training at which $\tau$ reaches unity.
	For notational convenience, we suppress the iteration index and write
	$\tau$ in place of $\tau(k)$ in subsequent expressions.
	
	\textbf{Curriculum-controlled component weights.}\;
	For each component~$m$, an easy-favoring weight and a hard-favoring
	weight are defined via exponential mappings:
	\begin{equation}
		w_m^{\mathrm{easy}} = \exp(-\beta\tilde{d}_m),\;
		w_m^{\mathrm{hard}} = \exp(-\beta(1 - \tilde{d}_m)),
		\label{eq:easy_hard_weights}
	\end{equation}
	where $\beta>0$ controls the intensity of discrimination between easy
	and hard components. The CL weight for component~$m$ is the
	$\tau$-controlled convex combination:
	\begin{equation}
		w_m^{\mathrm{CL}}(\tau)
		= (1-\tau)\,w_m^{\mathrm{easy}}
		+ \tau\,w_m^{\mathrm{hard}}.
		\label{eq:curriculum_weight}
	\end{equation}
	
	\textbf{Precision-based variance modulation.}\;
	Beyond the difficulty scores derived from residual magnitudes, the GMM
	provides component-level variance estimates $\sigma_m^2$ that encode the
	\emph{spread} of residuals within each cluster. We exploit the
	inverse variance (precision) as a complementary reliability indicator.
	The normalized precision factor for component~$m$ is
	\begin{equation}
		v_m
		= \frac{(\sigma_m^2+\epsilon)^{-1}}
		{\max_{1\leqslant l\leqslant K}(\sigma_l^2+\epsilon)^{-1}}
		\;\in\;(0,1].
		\label{eq:precision_factor}
	\end{equation}
	Components with small variance yield $v_m\!\approx\!1$,
	indicating tightly concentrated, reliably learnable residual clusters;
	components with large variance yield small $v_m$, reflecting dispersed,
	uncertain residuals that may destabilize early-stage optimization.
	To maintain curriculum compatibility, the precision factor is itself
	modulated by the curriculum parameter~$\tau$:
	\begin{equation}
		\tilde{v}_m(\tau) = (1-\tau)\,v_m + \tau.
		\label{eq:effective_precision}
	\end{equation}
	During early training ($\tau\!\approx\!0$), the full precision weighting
	suppresses high-variance components, prioritizing stable learning
	regions. As training progresses ($\tau\!\to\!1$), the modulation
	smoothly fades to unity, allowing all components---including
	previously suppressed uncertain regions---to contribute equally once
	the network has established a reliable baseline.
	
	The final component-level (comp) weight integrates both the curriculum
	difficulty weight and the precision modulation:
	\begin{equation}
		w_m^{\mathrm{comp}}(\tau)
		= w_m^{\mathrm{CL}}(\tau)\cdot\tilde{v}_m(\tau).
		\label{eq:combined_weight}
	\end{equation}
	This design achieves a \emph{dual curriculum}: along the difficulty
	dimension, the network progresses from easy to hard residual regions;
	along the reliability dimension, it progresses from well-conditioned
	(low-variance) to uncertain (high-variance) clusters. Both transitions
	are synchronized through the shared curriculum parameter~$\tau$.
	
	\textbf{Sample-level weights.}\;
	The per-sample weight for the $i$-th collocation point is obtained by
	aggregating the component-level weights through the posterior
	responsibilities:
	\begin{equation}
		w_i = \sum_{m=1}^{K}\gamma_{im}\,w_m^{\mathrm{comp}}(\tau),
		\label{eq:sample_weight}
	\end{equation}
	and is subsequently normalized to unit mean to preserve the overall
	loss magnitude:
	\begin{equation}
		w_i \leftarrow
		\frac{N_\Omega\cdot w_i}
		{\sum_{j=1}^{N_\Omega}w_j + \epsilon}.
		\label{eq:weight_normalize}
	\end{equation}
	The curriculum-weighted PDE loss evaluated at the current parameters $\theta$ is then
	\begin{equation}
		\mathcal{L}_{\mathrm{PDE}}^{w}(\theta)
		= \frac{1}{N_\Omega}\sum_{i=1}^{N_\Omega}
		w_i\,r_{\mathrm{PDE},i}^{2},
		\label{eq:weighted_pde_loss}
	\end{equation}
	where $r_{\mathrm{PDE},i}=r_{\mathrm{PDE}}(\boldsymbol{x}_i,t_i;\theta)$
	denotes the residual at the current parameters $\theta$
	(over which the gradient is taken), and the weights $\{w_i\}$ are
	computed from the snapshot residuals $\{r_i\}$ at the previous
	parameters $\theta_{k-1}$ and held fixed during the current gradient step.
	In practice, the GMM parameters and sample weights are re-estimated
	every $k_{\mathrm{upd}}$ iterations to adapt to the evolving residual
	landscape.
	
	\subsection{Self-Adaptive Loss Balancing via ReLoBRaLo}
	
	To mitigate gradient imbalance among different loss components, the
	CGMPINN framework optionally integrates the Relative Loss Balancing
	with Random Lookback (ReLoBRaLo) mechanism~\cite{Xiang26}.
	Suppose the total loss consists of $C$ components
	(e.g., $C=3$ for the PDE, boundary, and initial-condition terms),
	and let $L_c(k)$ denote the value of the $c$-th loss component
	($\mathcal{L}_{\mathrm{PDE}}^{w}$, $\mathcal{L}_{\mathrm{BC}}$,
	$\mathcal{L}_{\mathrm{IC}}$, respectively) at iteration~$k$.
	ReLoBRaLo maintains an exponential moving average (EMA)
	\begin{equation}
		\bar{L}_c(k)
		= \alpha\,\bar{L}_c(k\!-\!1)+(1\!-\!\alpha)\,L_c(k),
		\label{eq:relobralo_ema}
	\end{equation}
	with smoothing coefficient $\alpha\in(0,1)$.
	At each iteration, a reference loss $L_c^{\mathrm{ref}}(k)$ is defined as
	\begin{equation}
		L_c^{\mathrm{ref}}(k)=
		\begin{cases}
			\bar{L}_c(k-1), & \text{with probability } \rho,\\
			L_c(k'), & \text{otherwise},
		\end{cases}
		\label{eq:relobralo_ref}
	\end{equation}
	where $\rho\in(0,1)$ and $k'< k$ is a uniformly sampled
	historical iteration (random lookback).
	The adaptive weight is computed via a $\kappa$-scaled softmax over the
	relative ratios
	$\varrho_c(k)=L_c(k)/(L_c^{\mathrm{ref}}(k)+\epsilon)$:
	\begin{equation}
		\lambda_c(k)
		= C\cdot
		\frac{\exp\!\bigl(\varrho_c(k)/\kappa\bigr)}
		{\sum_{c'=1}^{C}
			\exp\!\bigl(\varrho_{c'}(k)/\kappa\bigr)},
		\label{eq:relobralo_weight}
	\end{equation}
	where $\kappa>0$ is the softmax temperature parameter. This mechanism
	automatically up-weights loss terms that decrease slowly relative to
	their historical trend, promoting balanced optimization without manual
	tuning.
	The resulting adaptive weights $\{\lambda_c(k)\}$ replace the fixed
	penalty weights $\lambda_{\mathrm{BC}},\lambda_{\mathrm{IC}}$ introduced
	in Eq.~\eqref{eq:standard_pinn_loss}. When ReLoBRaLo is not activated,
	uniform weights $\lambda_c=1$ are used.
	
	\subsection{Total Loss}
	
	Combining the curriculum-weighted PDE loss with the boundary and initial
	condition losses, the total training objective at iteration~$k$ is
	\begin{equation}
		\begin{split}
			\mathcal{L}_{\mathrm{total}}(\theta;k)
			=\;&\lambda_{\mathrm{PDE}}(k)\,
			\mathcal{L}_{\mathrm{PDE}}^{w}(\theta)
			+\lambda_{\mathrm{BC}}(k)\,
			\mathcal{L}_{\mathrm{BC}}(\theta)\\
			&+\lambda_{\mathrm{IC}}(k)\,
			\mathcal{L}_{\mathrm{IC}}(\theta),
		\end{split}
		\label{eq:total_loss}
	\end{equation}
	where $\mathcal{L}_{\mathrm{BC}}$ and $\mathcal{L}_{\mathrm{IC}}$
	are the standard MSE boundary and initial-condition losses:
	\begin{equation}
		\mathcal{L}_{\mathrm{BC}}(\theta)
		=\frac{1}{N_{\partial\Omega}}\sum_{j=1}^{N_{\partial\Omega}}r_{\mathrm{BC},j}^{2},\;
		\mathcal{L}_{\mathrm{IC}}(\theta)
		=\frac{1}{N_{\mathrm{IC}}}\sum_{l=1}^{N_{\mathrm{IC}}}r_{\mathrm{IC},l}^{2},
		\label{eq:bc_ic_loss}
	\end{equation}
	and $\lambda_{\mathrm{PDE}}(k)$, $\lambda_{\mathrm{BC}}(k)$,
	$\lambda_{\mathrm{IC}}(k)$ are the ReLoBRaLo adaptive weights
	(or set to unity when ReLoBRaLo is not activated).
	At each iteration, the network evaluates PDE, boundary, and initial
	residuals via automatic differentiation; every $k_{\mathrm{upd}}$
	iterations, the GMM is re-fitted to the current residuals and the
	sample weights $\{w_i\}$ are updated according to the curriculum
	parameter~$\tau(k)$.
	The parameters are then updated by a gradient-based optimizer
	(e.g., Adam followed by L-BFGS \cite{Haghighat45}).
	The framework is non-intrusive, requiring no modification to the PDE
	residual formulation or network architecture, and is applicable to both
	forward and inverse problems involving linear and nonlinear operators.

\subsection{Theoretical Properties of CGMPINN}
\label{sec:theory}

We now establish the main theoretical properties of CGMPINN.
The analysis targets the time-varying objective
$\mathcal{L}_{\mathrm{total}}(\theta;k)$ under idealized full-batch
gradient descent, so as to isolate the effect of the curriculum-guided
GMM reweighting and the induced objective drift.
The practical Adam$\to$L-BFGS strategy is regarded as an
effective solver for the same objective; a complete iterate-level theory
for this hybrid optimizer is left for future work.

For the following results, we denote by
$\mathcal{L}_{\mathrm{PDE}}(\theta)
=\frac{1}{N_\Omega}\sum_{i=1}^{N_\Omega}r_{\mathrm{PDE},i}^{2}$
the standard (unweighted) empirical PDE loss.
In the convergence statement (Theorem~\ref{thm:main_nonconvex}),
$K$ denotes the iteration horizon of the gradient descent analysis
(not the number of GMM components).
In the generalization bound (Theorem~\ref{thm:main_generalization}),
$\rho$ denotes the data-generating distribution over
$\Omega\times[0,T]$ (not the ReLoBRaLo parameter from
Section~\ref{sec:cgmpinn}).

\begin{theorem}[Uniform equivalence between weighted and standard PDE losses]
	\label{thm:main_equiv}
	Assume that the GMM variances satisfy
	$0<\underline{\sigma}^2\leqslant \sigma_m^2\leqslant \overline{\sigma}^2<\infty$
	for all mixture components.
	Then there exist constants $0<c_-\leqslant c_+<\infty$, depending only on
	$\beta$, $\epsilon$, $N_\Omega$, $\underline{\sigma}^2$, and
	$\overline{\sigma}^2$, such that for every $\theta$,
	\begin{equation}
		c_-\,\mathcal{L}_{\mathrm{PDE}}(\theta)
		\leqslant
		\mathcal{L}_{\mathrm{PDE}}^{w}(\theta)
		\leqslant
		c_+\,\mathcal{L}_{\mathrm{PDE}}(\theta).
		\label{eq:main_loss_equiv}
	\end{equation}
	Hence, the curriculum-weighted PDE loss is uniformly equivalent to the
	standard empirical PDE residual loss, and optimization of
	$\mathcal{L}_{\mathrm{PDE}}^{w}(\theta)$ remains faithful to the
	original PDE objective up to uniform constants.
\end{theorem}

\begin{remark}
	Theorem~\ref{thm:main_equiv} ensures that minimizing the curriculum-weighted objective $\mathcal{L}_{\mathrm{PDE}}^{w}(\theta)$ does not deviate from the original PDE residual minimization problem: any minimizer of the weighted loss is also an approximate minimizer of the standard loss, and vice versa.
\end{remark}

\begin{theorem}[Sublinear convergence of the time-varying total loss]
	\label{thm:main_nonconvex}
	Assume that, for each iteration $k$, the total loss
	$\mathcal{L}_{\mathrm{total}}(\theta;k)$ is $L$-smooth (i.e., its
	gradient is $L$-Lipschitz with Lipschitz constant $L>0$) and uniformly
	bounded from below, and that the inter-iteration drift induced by the
	curriculum update, GMM re-fitting, sample reweighting, and adaptive
	loss balancing is summable.
	Then, for the full-batch gradient descent iteration
	\begin{equation}
		\theta_{k+1}
		=
		\theta_k
		-
		\eta \nabla_{\theta}\mathcal{L}_{\mathrm{total}}(\theta_k;k),
		\quad
		\eta\in(0,1/L],
		\label{eq:main_gd_update}
	\end{equation}
	where $\eta>0$ denotes the step size,
	there exists a constant $C>0$, independent of $K$, such that
	\begin{equation}
		\min_{0\leqslant t\leqslant K}
		\left\|
		\nabla_{\theta}\mathcal{L}_{\mathrm{total}}(\theta_t;t)
		\right\|^2
		\leqslant
		\frac{C}{K+1}.
		\label{eq:main_sublinear_rate}
	\end{equation}
	In particular,
	\begin{equation}
		\left\|
		\nabla_{\theta}\mathcal{L}_{\mathrm{total}}(\theta_k;k)
		\right\|\to 0
		\quad \text{as } k\to\infty.
		\label{eq:main_stationarity}
	\end{equation}
\end{theorem}

\begin{remark}
	Theorem~\ref{thm:main_nonconvex} guarantees that, despite the time-varying nature of the CGMPINN objective (due to periodic GMM re-fitting, curriculum advancement, and adaptive loss balancing), gradient descent still converges to a stationary point at the standard $O(1/K)$ rate, provided the cumulative objective drift remains controlled.
\end{remark}

\begin{theorem}[Generalization bound for the weighted empirical PDE loss]
	\label{thm:main_generalization}
	Under the proof-oriented split-sample construction described in
	Appendix~\ref{app:theory}, assume that the PDE residual is uniformly
	bounded and that the induced weight function is bounded.
	Then, at a fixed refresh step and conditioned on the corresponding
	frozen weighting rule, with probability at least $1-\delta$, the
	population weighted PDE risk satisfies
	\begin{equation}
		\begin{aligned}
			\mathcal{E}_{\mathrm{PDE}}^{w}(\theta)
			\leqslant\;&
			\mathcal{L}_{\mathrm{PDE}}^{w}(\theta)
			+
			2W\,\mathfrak{R}_{N_\Omega}(\mathcal{F}) \\
			&+
			3BW\sqrt{\frac{\log(2/\delta)}{2N_\Omega}},
		\end{aligned}
		\label{eq:main_weighted_generalization}
	\end{equation}
	where $\mathcal{F}$ denotes the squared-residual function class,
	$B$ is the uniform residual bound,
	$W\geqslant 1$ is the uniform weight bound (i.e., $0\leqslant w_i\leqslant W$),
	and $\mathfrak{R}_{N_\Omega}(\mathcal{F})$ is the Rademacher complexity
	of $\mathcal{F}$.
	Moreover, the original unweighted population risk obeys
	\begin{equation}
		\begin{aligned}
			\mathcal{E}_{\mathrm{PDE}}(\theta)
			\leqslant\;&
			\mathcal{L}_{\mathrm{PDE}}^{w}(\theta)
			+
			2W\,\mathfrak{R}_{N_\Omega}(\mathcal{F}) \\
			&+
			3BW\sqrt{\frac{\log(2/\delta)}{2N_\Omega}}
			+
			B\|w-1\|_{L^1(\rho)}.
		\end{aligned}
		\label{eq:main_unweighted_generalization}
	\end{equation}
	Therefore, at a fixed refresh step, curriculum-based reweighting
	preserves statistical control of the induced weighted objective up to an
	explicit weighting-induced bias term relative to the original
	unweighted PDE risk.
\end{theorem}

\begin{remark}
	Theorem~\ref{thm:main_generalization} shows that the curriculum-weighted empirical loss generalizes to the population level at the standard rate, and the additional bias $B\|w-1\|_{L^1(\rho)}$ explicitly quantifies the cost of optimizing a reweighted surrogate instead of the original unweighted PDE risk.
\end{remark}

\begin{remark}[PL-based refinement]
	\label{rem:main_PL}
	Under an additional uniform Polyak--\L{}ojasiewicz condition, the
	time-varying total loss admits a stronger PL-based convergence theory:
	the homogeneous part decays linearly, a drift-neighborhood estimate
	holds under uniformly bounded drift, and exact convergence to the
	common reference level follows when the objective drift is summable.
\end{remark}

\begin{table*}[t]
	\centering
	\caption{Performance Comparison of Different Optimizers for Solving 1D Poisson Equation (with $\alpha_1 = 5.0$, $\alpha_2 = 3.0$, $s = 20.0$)}
	\label{tab:poisson_optimizer_performance}
	\begin{tabular}{lccccc}
		\toprule
		Optimizer & $e_{\text{Loss}}$ & $e_2$ & $\text{Relative } e_2$ & $e_\infty$ & CPU (s) \\
		\midrule
		Adam               & $2.30e\text{+}1$ & $1.70e\text{+}0$ & $1.57e\text{+}0$ & $2.93e\text{+}0$ & $199.0$   \\
		L-BFGS             & $\mathbf{8.24e\text{-}4}$ & $3.15e\text{-}4$ & $2.91e\text{-}4$ & $4.90e\text{-}4$ & $1099.4$  \\
		Adam$\to$L-BFGS    & $9.60e\text{-}4$ & $\mathbf{1.96e\text{-}4}$ & $\mathbf{1.81e\text{-}4}$ & $\mathbf{3.74e\text{-}4}$ & $895.9$   \\
		\bottomrule
	\end{tabular}
\end{table*}

\begin{table*}[t]
	\centering
	\caption{Performance Comparison of Different Methods for Solving 1D Poisson Equation (with $\alpha_1 = 5.0$, $\alpha_2 = 3.0$, $s = 20.0$)}
	\label{tab:1d_poisson_pinn_methods_performance}
	\begin{tabular}{lccccc}
		\toprule
		Method & $e_{\text{Loss}}$ & $e_2$ & $\text{Relative } e_2$ & $e_\infty$ & CPU (s) \\
		\midrule
		PINN               & $1.90e\text{-}3$ & $8.83e\text{-}3$ & $8.15e\text{-}3$ & $9.19e\text{-}3$ & $1018.2$   \\
		lbPINN             & $2.95e\text{+}0$ & $4.57e\text{+}0$ & $4.22e\text{+}0$ & $7.88e\text{+}0$ & $1205.6$   \\
		gPINN              & $1.02e\text{-}1$ & $1.88e\text{-}2$ & $1.74e\text{-}2$ & $2.47e\text{-}2$ & $5711.8$   \\
		LNN-PINN           & $2.31e\text{-}3$ & $4.50e\text{-}4$ & $4.15e\text{-}4$ & $7.89e\text{-}4$ & $1422.0$   \\
		STAR-PINN          & $6.87e\text{-}2$ & $1.23e\text{-}3$ & $1.14e\text{-}3$ & $2.05e\text{-}3$ & $1332.4$   \\
		CGMPINN            & $\mathbf{9.60e\text{-}4}$ & $\mathbf{1.96e\text{-}4}$ & $\mathbf{1.81e\text{-}4}$ & $\mathbf{3.74e\text{-}4}$ & $895.9$   \\
		\bottomrule
	\end{tabular}
\end{table*}

Detailed assumptions, explicit constants, and complete proofs of
Theorems~\ref{thm:main_equiv}--\ref{thm:main_generalization} and
Remark~\ref{rem:main_PL} are deferred to Appendix~\ref{app:theory}.

\section{Numerical Experiments and Comparisons}
\label{sec:numerical_experiments}

This section systematically evaluates the proposed CGMPINN framework on a suite of benchmark PDEs spanning different types: elliptic (1D and 2D Poisson), parabolic (heat equation), hyperbolic (damped wave equation), advection-dominated (advection-diffusion equation), and nonlinear reaction-diffusion (Fisher-KPP equation). These problems are selected to cover a range of difficulties commonly encountered in practice, including high-gradient solutions, steep wavefronts, damped oscillatory dynamics, and nonlinear traveling waves.

For each benchmark, we first conduct an optimizer study comparing Adam, L-BFGS, and a two-stage Adam$\to$L-BFGS strategy \cite{Haghighat45} within the CGMPINN framework, in order to identify the most effective training configuration. We then compare CGMPINN against five representative PINN variants: the canonical PINN \cite{PINN_Raissi1}, lbPINN \cite{Xiang26}, gPINN \cite{Yu28}, LNN-PINN \cite{Tao29}, and STAR-PINN \cite{Dodge30}. All methods employ the same fully-connected neural network architecture as the solution approximator and are implemented in PyTorch to ensure a fair comparison.
Note that in the PDE formulations below, problem-specific parameters (e.g., $\gamma$, $\nu$, $r$) refer to physical coefficients of the respective equations and should not be confused with the method parameters defined in Section~\ref{sec:cgmpinn}.

In all benchmarks, the source term $f$, boundary data $g$, and initial conditions are determined by substituting the exact solution into the governing equation and the corresponding boundary/initial conditions.

The evaluation metrics are defined as follows: the final training loss $e_{\text{Loss}} = \mathcal{L}_{\mathrm{total}}(\theta^*)$, the absolute $L_2$ error $e_2 = \|u - \hat{u}\|_2$, the relative $L_2$ error $e_2/\|u\|_2$, and the maximum absolute error $e_\infty = \|u - \hat{u}\|_\infty$, where $\theta^*$ denotes the trained network parameters, and $u$ and $\hat{u}$ denote the exact and predicted solutions evaluated on the test points, respectively. Additionally, an ablation study comparing GMMPINN (GMM-based weighting only), CLPINN (curriculum learning only), and the full CGMPINN is conducted to disentangle the contributions of the individual components; the results are reported in Appendix~\ref{app:ablation}.

\begin{table*}[t]
	\centering
	\caption{Performance Comparison of Different Optimizers for Solving 2D Poisson Equation (with $\beta_1 = 3.0$, $\beta_2 = 2.0$)}
	\label{tab:2d_poisson_beta_optimizer_performance}
	\begin{tabular}{lccccc}
		\toprule
		Optimizer & $e_{\text{Loss}}$ & $e_2$ & $\text{Relative } e_2$ & $e_\infty$ & CPU (s) \\
		\midrule
		Adam               & $8.23e\text{-}2$ & $3.97e\text{-}2$ & $4.99e\text{-}2$ & $2.46e\text{-}1$ & $441.4$   \\
		L-BFGS             & $3.93e\text{-}3$ & $3.64e\text{-}3$ & $4.56e\text{-}3$ & $3.54e\text{-}2$ & $1914.7$  \\
		Adam$\to$L-BFGS    & $\mathbf{7.78e\text{-}5}$ & $\mathbf{4.65e\text{-}4}$ & $\mathbf{5.83e\text{-}4}$ & $\mathbf{2.95e\text{-}3}$ & $1453.2$  \\
		\bottomrule
	\end{tabular}
\end{table*}

\begin{table*}[t]
	\centering
	\caption{Performance Comparison of Different Methods for Solving 2D Poisson Equation (with $\beta_1 = 3.0$, $\beta_2 = 2.0$)}
	\label{tab:2d_poisson_pinn_methods_performance}
	\begin{tabular}{lccccc}
		\toprule
		Method & $e_{\text{Loss}}$ & $e_2$ & $\text{Relative } e_2$ & $e_\infty$ & CPU (s) \\
		\midrule
		PINN               & $1.06e\text{-}4$ & $1.08e\text{-}3$ & $1.35e\text{-}3$ & $7.28e\text{-}3$ & $1376.5$   \\
		lbPINN             & $2.03e\text{-}2$ & $4.07e\text{-}3$ & $5.11e\text{-}3$ & $2.45e\text{-}2$ & $1509.5$   \\
		gPINN              & $2.41e\text{-}4$ & $1.60e\text{-}3$ & $2.01e\text{-}3$ & $9.93e\text{-}3$ & $6876.9$   \\
		LNN-PINN           & $7.93e\text{-}5$ & $8.12e\text{-}4$ & $1.02e\text{-}3$ & $9.36e\text{-}3$ & $2775.1$   \\
		STAR-PINN          & $\mathbf{7.30e\text{-}5}$ & $5.19e\text{-}4$ & $6.51e\text{-}4$ & $5.37e\text{-}3$ & $2339.2$   \\
		CGMPINN            & $7.78e\text{-}5$ & $\mathbf{4.65e\text{-}4}$ & $\mathbf{5.83e\text{-}4}$ & $\mathbf{2.95e\text{-}3}$ & $1453.2$   \\
		\bottomrule
	\end{tabular}
\end{table*}

\subsection{1D Poisson Equation}
\label{subsec:1d_poisson}

We begin with a 1D Poisson equation featuring a high-gradient analytical solution, which serves as a standard test for capturing sharp spatial variations. The problem is defined as
\begin{equation}
	\begin{cases}
		u_{xx}(x) = f(x),\quad x \in \Omega, \\
		u(x) = g(x),\quad x \in \partial\Omega,
	\end{cases}
	\label{eq:1d_poisson}
\end{equation}
with the analytical solution
\begin{equation}
	u(x) = \sin(\alpha_1 \pi x)\cos(\alpha_2 \pi x) + \tanh(sx),
	\label{eq:1d_poisson_solution}
\end{equation}
where $\Omega=[0,1]$, $\alpha_1 = 5.0$, $\alpha_2 = 3.0$, and $s = 20.0$ is the steepness parameter. The combination of multi-frequency sinusoidal components and a steep hyperbolic tangent transition makes this problem particularly challenging for neural network-based solvers. All methods in this subsection employ a four-hidden-layer network with 50 neurons per layer, trained on 1{,}500 randomly sampled interior collocation points and evaluated on 200 uniformly distributed test points.

Table~\ref{tab:poisson_optimizer_performance} compares the three optimization strategies. The vanilla Adam optimizer fails to converge to an acceptable solution, producing errors several orders of magnitude larger than the two L-BFGS-based strategies. While pure L-BFGS achieves the lowest training loss ($8.24\times10^{-4}$), the Adam$\to$L-BFGS strategy yields the best generalization performance, reducing $e_2$ by 37.8\% and $e_\infty$ by 23.7\% compared to pure L-BFGS. Moreover, the two-stage strategy is 18.5\% faster than pure L-BFGS, indicating that Adam pre-training provides a favorable initialization that accelerates subsequent L-BFGS convergence. The corresponding training loss curves and prediction results are shown in Figure~\ref{fig:1d_poisson_optimizer}. Based on these consistent findings, the Adam$\to$L-BFGS strategy is adopted for all subsequent experiments.

\begin{figure}[htbp]
	\centering
	\includegraphics[width=0.5\textwidth]{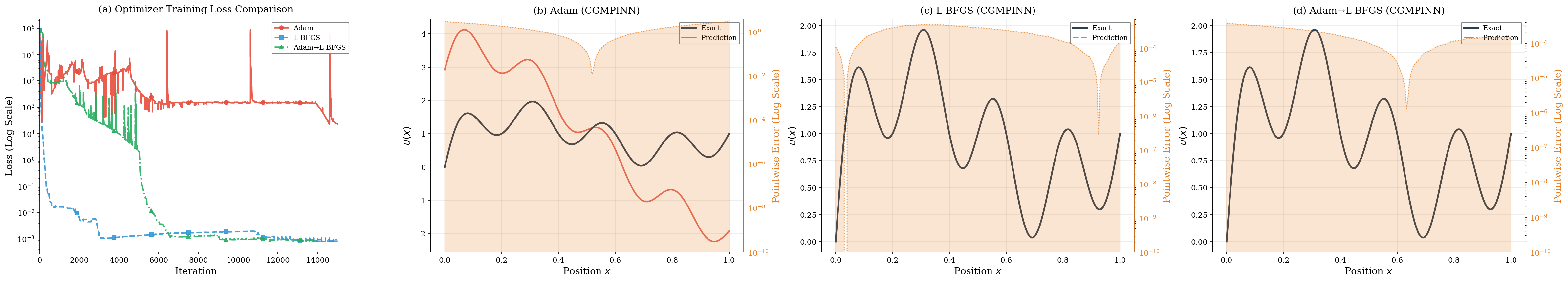}
	\caption{Optimizer comparison for 1D Poisson equation: (a) Training loss curves with different optimizers; (b)--(d) Predictions of CGMPINN with Adam, L-BFGS, and Adam$\to$L-BFGS optimizers respectively, compared with the exact solution (the pointwise error in log scale is shown as the auxiliary y-axis).}
	\label{fig:1d_poisson_optimizer}
\end{figure}

Table~\ref{tab:1d_poisson_pinn_methods_performance} presents the comparison of CGMPINN against five baseline methods. CGMPINN achieves the best performance across all error metrics. In terms of $e_2$, CGMPINN attains $1.96\times10^{-4}$, representing a reduction of 97.8\% relative to the canonical PINN ($8.83\times10^{-3}$), 56.4\% relative to LNN-PINN ($4.50\times10^{-4}$), and 84.1\% relative to STAR-PINN ($1.23\times10^{-3}$). The $e_\infty$ of CGMPINN ($3.74\times10^{-4}$) is also the lowest among all methods, indicating uniformly accurate predictions across the domain. It is worth noting that lbPINN fails to converge for this problem, with errors exceeding those of the canonical PINN by orders of magnitude, and gPINN also underperforms the standard PINN despite incurring a 5.6$\times$ higher computational cost. In contrast, CGMPINN achieves the best accuracy while maintaining the lowest CPU time among all methods, demonstrating that the curriculum-guided weighting mechanism improves accuracy without introducing significant computational overhead. The training loss curves and error distributions are further illustrated in Figure~\ref{fig:1d_poisson_methods}.

\begin{figure}[htbp]
	\centering
	\includegraphics[width=0.5\textwidth]{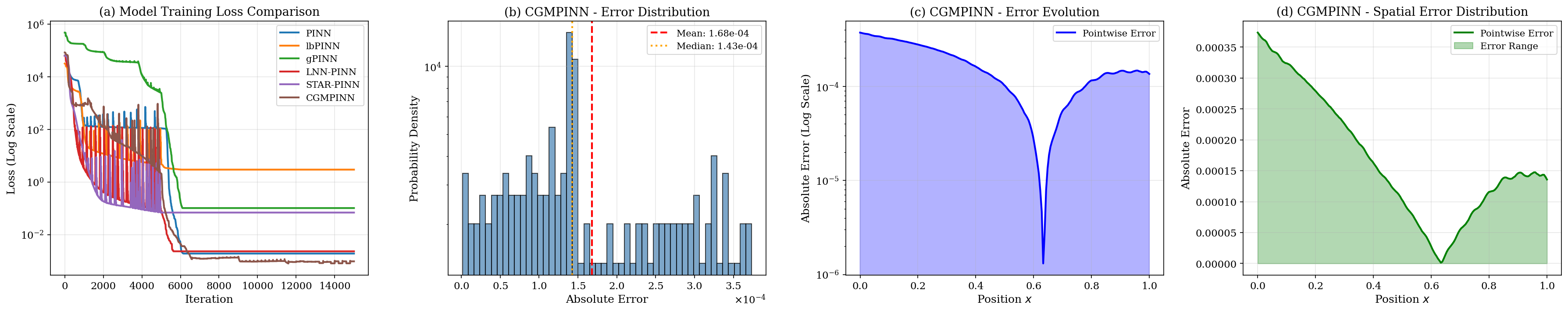}
	\caption{Performance comparison of different PINN variants for 1D Poisson equation: (a) Training loss curves of different models; (b) Absolute error probability density distribution of CGMPINN; (c) Pointwise error (log scale) of CGMPINN in the spatial domain; (d) Spatial absolute error distribution and error range of CGMPINN.}
	\label{fig:1d_poisson_methods}
\end{figure}

\subsection{2D Poisson Equation}
\label{subsec:2d_poisson}

To assess scalability to higher spatial dimensions, we consider a 2D Poisson equation on a rectangular domain with Dirichlet boundary conditions:
\begin{equation}
	\begin{cases}
		u_{xx}(x,y) + u_{yy}(x,y) = f(x,y),\quad (x,y)\in\Omega, \\
		u(x,y)=g(x,y),\quad (x,y)\in\partial\Omega,
	\end{cases}
	\label{eq:2d_poisson}
\end{equation}
with the analytical solution
\begin{equation}
	u(x, y) = \sin(\beta_1 \pi x)\sin(\beta_2 \pi y) + e^{-x^2 - y^2},
	\label{eq:2d_poisson_solution}
\end{equation}
where $\Omega=[0,1]^{2}$, $\beta_1 = 3.0$ and $\beta_2 = 2.0$. The multi-frequency oscillations coupled with a localized Gaussian component introduce anisotropic gradients in both spatial directions. All methods in this subsection use a four-hidden-layer network with 50 neurons per layer, trained on 2{,}000 randomly sampled interior collocation points and 250 boundary points, and evaluated on 10{,}000 uniformly distributed test points.

Table~\ref{tab:2d_poisson_beta_optimizer_performance} reports the optimizer comparison results. Consistent with the 1D case, the Adam$\to$L-BFGS strategy achieves the best performance across all metrics. Compared to pure L-BFGS, it reduces $e_{\text{Loss}}$ by 98.0\%, $e_2$ by 87.2\%, and $e_\infty$ by 91.7\%, while also reducing the training time by 24.1\%. The Adam optimizer again shows limited convergence capability, with errors approximately one order of magnitude higher than the L-BFGS-based strategies. The training loss curves and prediction results are shown in Figure~\ref{fig:2d_poisson_optimizer}.

\begin{figure}[htbp]
	\centering
	\includegraphics[width=0.5\textwidth]{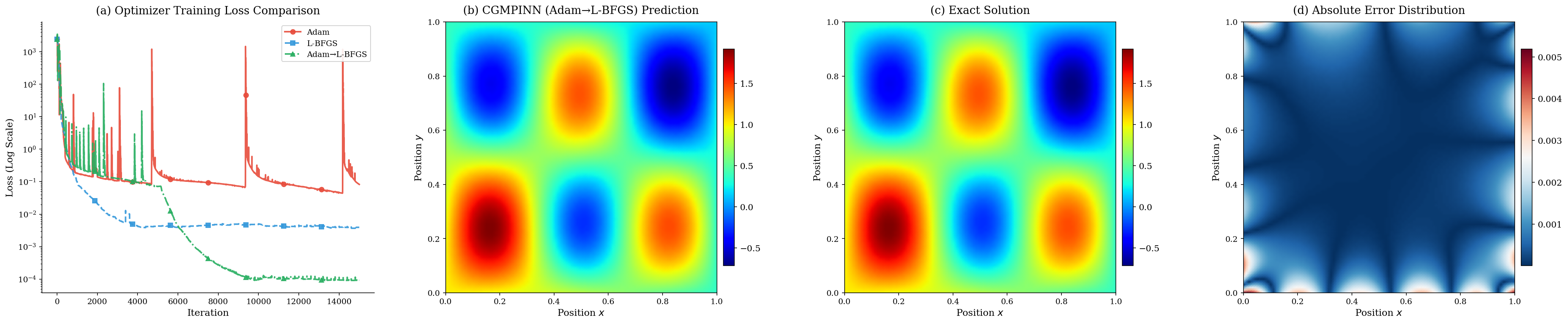}
	\caption{Optimizer comparison for 2D Poisson equation: (a) Training loss curves of different optimizers; (b) Prediction result of CGMPINN with Adam$\to$L-BFGS optimizer; (c) Exact solution of the equation; (d) Absolute error distribution of the prediction.}
	\label{fig:2d_poisson_optimizer}
\end{figure}

\begin{table*}[t]
	\centering
	\caption{Performance Comparison of Different Optimizers for Solving Heat Equation (with $\alpha_1 = 1.0$, $\alpha_2 = 2.0$, $s = 10.0$)}
	\label{tab:heat_equation_optimizer_performance}
	\begin{tabular}{lccccc}
		\toprule
		Optimizer & $e_{\text{Loss}}$ & $e_2$ & $\text{Relative } e_2$ & $e_\infty$ & CPU (s) \\
		\midrule
		Adam               & $1.38e\text{-}2$ & $1.63e\text{-}2$ & $1.43e\text{-}2$ & $1.21e\text{-}1$ & $292.9$   \\
		L-BFGS             & $8.69e\text{-}5$ & $1.18e\text{-}3$ & $1.03e\text{-}3$ & $7.63e\text{-}3$ & $1517.3$  \\
		Adam$\to$L-BFGS    & $\mathbf{1.53e\text{-}5}$ & $\mathbf{4.04e\text{-}4}$ & $\mathbf{3.54e\text{-}4}$ & $\mathbf{2.86e\text{-}3}$ & $1430.0$    \\
		\bottomrule
	\end{tabular}
\end{table*}

\begin{table*}[t]
	\centering
	\caption{Performance Comparison of Different Methods for Solving Heat Equation (with $\alpha_1 = 1.0$, $\alpha_2 = 2.0$, $s = 10.0$)}
	\label{tab:heat_equation_pinn_methods_performance}
	\begin{tabular}{lccccc}
		\toprule
		Method & $e_{\text{Loss}}$ & $e_2$ & $\text{Relative } e_2$ & $e_\infty$ & CPU (s) \\
		\midrule
		PINN               & $3.59e\text{-}5$ & $1.46e\text{-}3$ & $1.28e\text{-}3$ & $1.93e\text{-}2$ & $1342.8$   \\
		lbPINN             & $1.58e\text{-}2$ & $3.22e\text{-}3$ & $2.83e\text{-}3$ & $3.44e\text{-}2$ & $1134.1$   \\
		gPINN              & $1.56e\text{-}4$ & $3.11e\text{-}3$ & $2.73e\text{-}3$ & $3.41e\text{-}2$ & $5331.7$   \\
		LNN-PINN           & $3.09e\text{-}5$ & $9.57e\text{-}4$ & $8.39e\text{-}4$ & $1.08e\text{-}2$ & $2333.2$   \\
		STAR-PINN          & $9.96e\text{-}5$ & $1.63e\text{-}3$ & $1.43e\text{-}3$ & $1.88e\text{-}2$ & $1776.6$   \\
		CGMPINN            & $\mathbf{1.53e\text{-}5}$ & $\mathbf{4.04e\text{-}4}$ & $\mathbf{3.54e\text{-}4}$ & $\mathbf{2.86e\text{-}3}$ & $1430.0$   \\
		\bottomrule
	\end{tabular}
\end{table*}

Table~\ref{tab:2d_poisson_pinn_methods_performance} compares CGMPINN with five baseline methods on the 2D Poisson equation. From the perspective of computational efficiency, CGMPINN achieves a favorable trade-off between runtime and solution quality. Its training cost is 1453.2\,s, which is close to that of the standard PINN (1376.5\,s) and substantially lower than those of LNN-PINN (2775.1\,s), STAR-PINN (2339.2\,s), and especially gPINN (6876.9\,s). Despite this relatively low computational cost, CGMPINN still attains the best accuracy among all compared methods, yielding the lowest values of $e_2$ ($4.65\times10^{-4}$), relative $e_2$ ($5.83\times10^{-4}$), and $e_\infty$ ($2.95\times10^{-3}$). In particular, compared with the standard PINN, CGMPINN reduces $e_2$ by 56.9\% and $e_\infty$ by 59.5\%; compared with STAR-PINN, it further improves $e_2$ and $e_\infty$ by 10.4\% and 45.1\%, respectively. Although CGMPINN does not achieve the minimum training loss, its $e_{\text{Loss}}$ remains comparable to that of STAR-PINN. Overall, these results suggest that CGMPINN achieves a more favorable balance between computational efficiency and predictive accuracy on this benchmark, while also demonstrating its effectiveness in multi-dimensional settings. The training loss curves and error distributions are further illustrated in Figure~\ref{fig:2d_poisson_methods}.

\begin{figure}[htbp]
	\centering
	\includegraphics[width=0.5\textwidth]{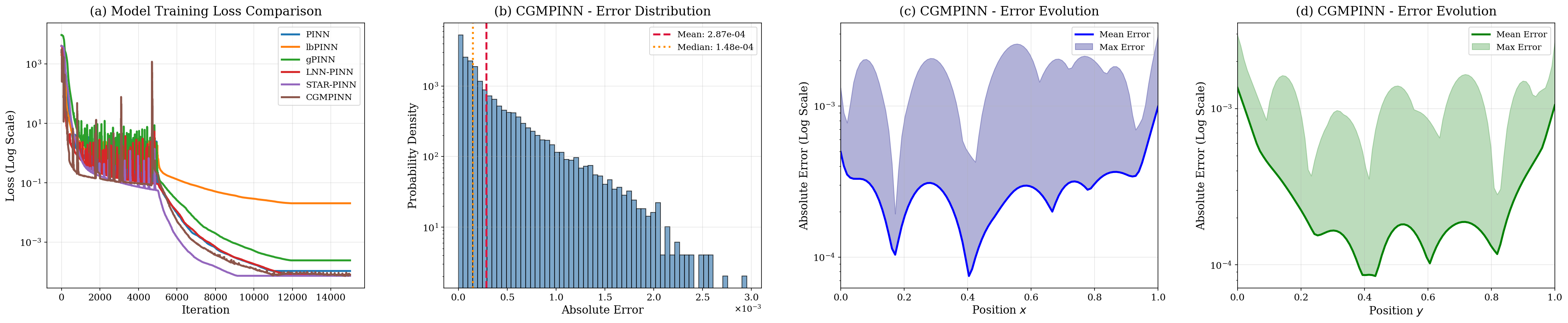}
	\caption{Performance comparison of different PINN variants for 2D Poisson equation: (a) Training loss curves of different models; (b) Absolute error probability density distribution of CGMPINN; (c) Mean and max absolute error of CGMPINN in the x spatial direction; (d) Mean and max absolute error of CGMPINN in the y spatial direction.}
	\label{fig:2d_poisson_methods}
\end{figure}

\begin{table*}[t]
	\centering
	\caption{Performance Comparison of Different Optimizers for Solving Damped Wave Equation (with $\alpha_1 = 1.0$, $\alpha_2 = 1.0$, $\gamma = 0.1$)}
	\label{tab:1d_damped_wave_optimizer_performance}
	\begin{tabular}{lccccc}
		\toprule
		Optimizer & $e_{\text{Loss}}$ & $e_2$ & $\text{Relative } e_2$ & $e_\infty$ & CPU (s) \\
		\midrule
		Adam               & $1.81e\text{-}4$ & $4.59e\text{-}3$ & $9.64e\text{-}3$ & $1.50e\text{-}2$ & $535.7$   \\
		L-BFGS             & $1.33e\text{-}5$ & $1.18e\text{-}3$ & $2.49e\text{-}3$ & $6.11e\text{-}3$ & $2486.0$  \\
		Adam$\to$L-BFGS    & $\mathbf{1.66e\text{-}6}$ & $\mathbf{3.37e\text{-}4}$ & $\mathbf{7.07e\text{-}4}$ & $\mathbf{2.20e\text{-}3}$ & $678.0$   \\
		\bottomrule
	\end{tabular}
\end{table*}

\begin{table*}[t]
	\centering
	\caption{Performance Comparison of Different Methods for Solving Damped Wave Equation (with $\alpha_1 = 1.0$, $\alpha_2 = 1.0$, $\gamma = 0.1$)}
	\label{tab:1d_damped_wave_pinn_methods_performance}
	\begin{tabular}{lccccc}
		\toprule
		Method & $e_{\text{Loss}}$ & $e_2$ & $\text{Relative } e_2$ & $e_\infty$ & CPU (s) \\
		\midrule
		PINN               & $1.17e\text{-}5$ & $1.14e\text{-}3$ & $2.40e\text{-}3$ & $3.81e\text{-}3$ & $718.2$   \\
		lbPINN             & $1.62e\text{-}3$ & $4.65e\text{-}4$ & $9.76e\text{-}4$ & $3.48e\text{-}3$ & $1881.3$   \\
		gPINN              & $7.74e\text{-}6$ & $9.48e\text{-}4$ & $1.99e\text{-}3$ & $3.00e\text{-}3$ & $2971.8$   \\
		LNN-PINN           & $\mathbf{1.54e\text{-}6}$ & $4.66e\text{-}4$ & $9.79e\text{-}4$ & $2.30e\text{-}3$ & $1164.5$   \\
		STAR-PINN          & $3.22e\text{-}5$ & $2.53e\text{-}3$ & $5.31e\text{-}3$ & $7.14e\text{-}3$ & $916.4$   \\
		CGMPINN            & $1.66e\text{-}6$ & $\mathbf{3.37e\text{-}4}$ & $\mathbf{7.07e\text{-}4}$ & $\mathbf{2.20e\text{-}3}$ & $678.0$   \\
		\bottomrule
	\end{tabular}
\end{table*}

\subsection{Heat Equation}
\label{subsec:heat_equation}

We next consider a time-dependent parabolic problem: the 1D heat equation with a localized steep gradient induced by a hyperbolic tangent term. The problem reads
\begin{equation}
	\begin{cases}
		u_t(x,t) = u_{xx}(x,t) + f(x,t),\quad (x,t) \in \Omega \times [0,1], \\
		u(x,0) = u_0(x),\quad x \in \Omega, \\
		u(x,t) = g(x,t),\quad (x,t) \in \partial\Omega \times [0,1],
	\end{cases}
	\label{eq:heat_equation}
\end{equation}
with the analytical solution
\begin{equation}
	u(x,t) = \left( \sin(\alpha_1 \pi x) + \tanh(sx) \right) \sin(\alpha_2 \pi t),
	\label{eq:heat_equation_solution}
\end{equation}
where $\Omega=[0,1]$, $\alpha_1 = 1.0$, $\alpha_2 = 2.0$, and $s = 10.0$. The temporal oscillation modulated by a spatially localized steep gradient tests the ability of neural solvers to resolve both temporal dynamics and spatial fine structure simultaneously. All methods in this subsection employ a four-hidden-layer network with 50 neurons per layer, trained on 1{,}500 randomly sampled interior collocation points, 300 boundary points, and 300 initial condition points, and evaluated on 10{,}000 uniformly distributed test points.

Table~\ref{tab:heat_equation_optimizer_performance} reports the optimizer comparison. The Adam$\to$L-BFGS strategy again achieves the best results across all metrics, reducing $e_{\text{Loss}}$ by 82.4\% and $e_2$ by 65.8\% compared to pure L-BFGS, while requiring 5.8\% less training time. The Adam optimizer yields errors more than an order of magnitude higher than the L-BFGS-based strategies, consistent with observations on the Poisson problems. The training loss curves and prediction results are shown in Figure~\ref{fig:heat_equation_optimizer}.

\begin{figure}[htbp]
	\centering
	\includegraphics[width=0.5\textwidth]{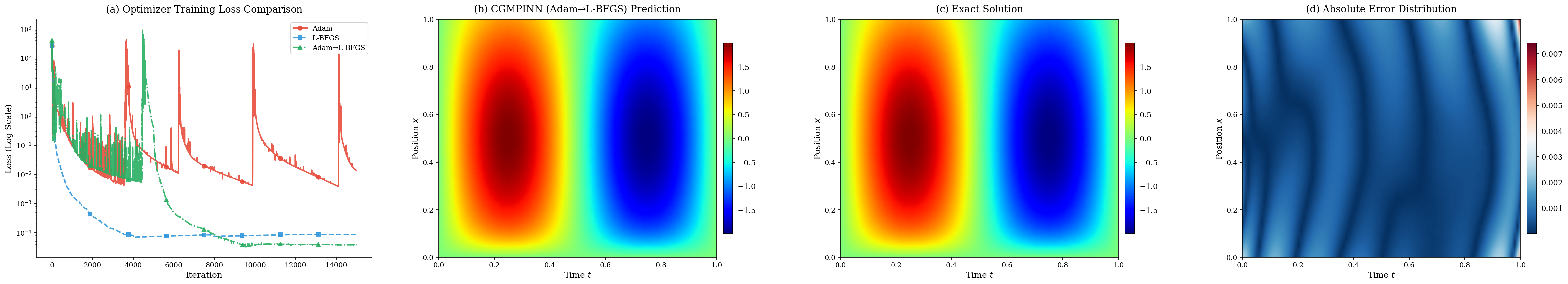}
	\caption{Optimizer comparison for 1D heat equation: (a) Training loss curves of different optimizers; (b) Prediction result of CGMPINN with Adam$\to$L-BFGS optimizer; (c) Exact solution of the equation; (d) Absolute error distribution of the prediction.}
	\label{fig:heat_equation_optimizer}
\end{figure}

Table~\ref{tab:heat_equation_pinn_methods_performance} compares CGMPINN with five baseline methods on the heat equation. CGMPINN achieves the best accuracy among all compared methods, attaining the lowest values of $e_2$, relative $e_2$, and $e_\infty$. In particular, its $e_2$ is $4.04\times10^{-4}$, corresponding to reductions of 72.3\% and 57.8\% relative to the standard PINN and LNN-PINN, respectively. The smallest $e_\infty$ ($2.86\times10^{-3}$) further indicates that CGMPINN provides more uniformly accurate predictions, with reductions of 85.2\% and 73.5\% compared with the standard PINN and LNN-PINN, respectively. In addition, CGMPINN also attains the lowest training loss ($1.53\times10^{-5}$), suggesting more effective optimization on this benchmark. Although its CPU time (1430.0\,s) is slightly higher than that of the standard PINN (1342.8\,s), it remains substantially lower than those of LNN-PINN (2333.2\,s), STAR-PINN (1776.6\,s), and gPINN (5331.7\,s). Overall, these results indicate that CGMPINN achieves a favorable balance between predictive accuracy and computational cost on this benchmark. The training loss curves and error distributions are further illustrated in Figure~\ref{fig:heat_equation_methods}.

\begin{figure}[htbp]
	\centering
	\includegraphics[width=0.5\textwidth]{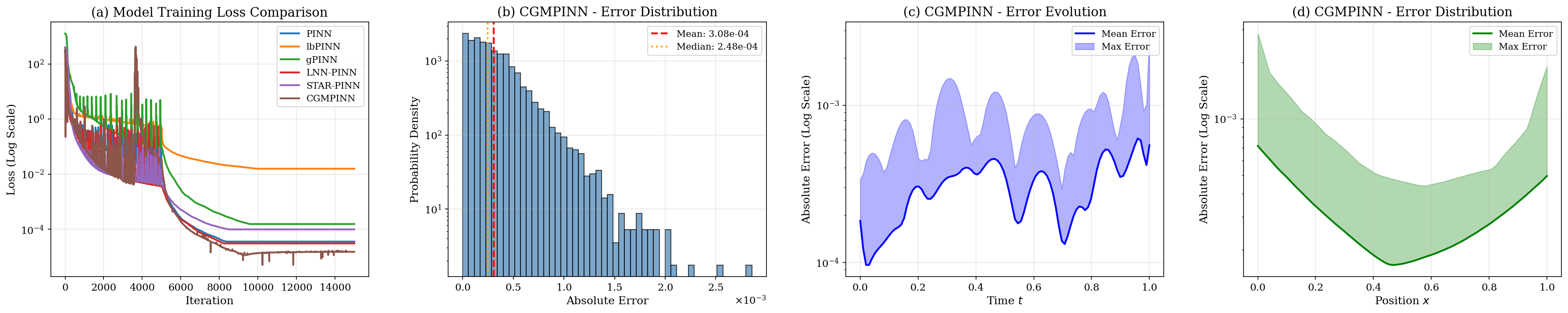}
	\caption{Performance comparison of different PINN variants for 1D heat equation: (a) Training loss curves of different models; (b) Absolute error probability density distribution of CGMPINN; (c) Temporal evolution of the mean and max absolute error of CGMPINN; (d) Mean and max absolute error of CGMPINN in the x spatial direction.}
	\label{fig:heat_equation_methods}
\end{figure}

\begin{table*}[t]
	\centering
	\caption{Performance Comparison of Different Optimizers for Solving Advection-Diffusion Equation (with $a = 1.0$, $\nu = 10^{-2}$)}
	\label{tab:advection_diffusion_optimizer_performance}
	\begin{tabular}{lccccc}
		\toprule
		Optimizer & $e_{\text{Loss}}$ & $e_2$ & $\text{Relative } e_2$ & $e_\infty$ & CPU (s) \\
		\midrule
		Adam               & $2.44e\text{-}4$ & $7.35e\text{-}3$ & $1.09e\text{-}2$ & $1.33e\text{-}2$ & $434.7$   \\
		L-BFGS             & $1.84e\text{-}6$ & $5.27e\text{-}4$ & $7.82e\text{-}4$ & $1.33e\text{-}3$ & $512.2$  \\
		Adam$\to$L-BFGS    & $\mathbf{1.54e\text{-}6}$ & $\mathbf{2.57e\text{-}4}$ & $\mathbf{3.82e\text{-}4}$ & $\mathbf{7.14e\text{-}4}$ & $532.4$   \\
		\bottomrule
	\end{tabular}
\end{table*}

\begin{table*}[t]
	\centering
	\caption{Performance Comparison of Different Methods for Solving Advection-Diffusion Equation (with $a = 1.0$, $\nu = 10^{-2}$)}
	\label{tab:advection_diffusion_pinn_methods_performance}
	\begin{tabular}{lccccc}
		\toprule
		Method & $e_{\text{Loss}}$ & $e_2$ & $\text{Relative } e_2$ & $e_\infty$ & CPU (s) \\
		\midrule
		PINN               & $4.28e\text{-}6$ & $7.14e\text{-}4$ & $1.06e\text{-}3$ & $2.07e\text{-}3$ & $446.6$   \\
		lbPINN             & $2.76e\text{-}3$ & $4.44e\text{-}4$ & $6.60e\text{-}4$ & $1.16e\text{-}3$ & $2018.2$   \\
		gPINN              & $5.48e\text{-}6$ & $5.69e\text{-}4$ & $8.44e\text{-}4$ & $1.35e\text{-}3$ & $1885.6$   \\
		LNN-PINN           & $1.66e\text{-}5$ & $7.97e\text{-}4$ & $1.18e\text{-}3$ & $2.79e\text{-}3$ & $678.6$   \\
		STAR-PINN          & $3.56e\text{-}6$ & $4.41e\text{-}4$ & $6.55e\text{-}4$ & $1.57e\text{-}3$ & $623.6$   \\
		CGMPINN            & $\mathbf{1.54e\text{-}6}$ & $\mathbf{2.57e\text{-}4}$ & $\mathbf{3.82e\text{-}4}$ & $\mathbf{7.14e\text{-}4}$ & $532.4$   \\
		\bottomrule
	\end{tabular}
\end{table*}

\subsection{Damped Wave Equation}
\label{subsec:damped_wave_equation}

We proceed to a second-order hyperbolic problem: the 1D damped wave equation, which models wave propagation with energy dissipation. The governing equation is
\begin{equation}
	\begin{cases}
		u_{tt} + 2\gamma u_t - c^2 u_{xx} = 0,\quad (x,t) \in \Omega \times [0,1], \\
		u(x,0) = u_0(x),\quad x \in \Omega, \\
		u_t(x,0) = v_0(x),\quad x \in \Omega, \\
		u(x,t) = g(x,t),\quad (x,t) \in \partial\Omega \times [0,1],
	\end{cases}
	\label{eq:damped_wave_equation}
\end{equation}
where $\Omega=[0,1]$, $\gamma>0$ is the damping coefficient, $c$ is the wave speed, and $u_0(x)$ and $v_0(x)$ denote the initial displacement and velocity profiles, respectively. The analytical solution is
\begin{equation}
	u(x,t) = e^{-\gamma t} \sin(\alpha_1 \pi x) \cos(\alpha_2 \pi t),
	\label{eq:damped_wave_solution}
\end{equation}
with $c = \sqrt{\gamma^2 + (\alpha_2 \pi)^2}/(\alpha_1 \pi)$, and parameters $\alpha_1 = 1.0$, $\alpha_2 = 1.0$, $\gamma = 0.1$. The exponentially decaying oscillatory behavior requires the solver to accurately resolve both the spatial modal structure and the temporal amplitude decay. All methods in this subsection employ a four-hidden-layer network with 50 neurons per layer, trained on 2{,}000 randomly sampled interior collocation points, 300 boundary points, and 300 initial condition points, and evaluated on 10{,}000 uniformly distributed test points.

Table~\ref{tab:1d_damped_wave_optimizer_performance} shows the optimizer comparison. The Adam$\to$L-BFGS strategy delivers the most pronounced advantage on this problem, reducing $e_{\text{Loss}}$ by 87.5\%, $e_2$ by 71.4\%, and $e_\infty$ by 64.0\% relative to pure L-BFGS, while requiring only 27.3\% of the L-BFGS training time. This significant efficiency gain suggests that for second-order time-dependent PDEs, Adam pre-training is especially effective at navigating the loss landscape toward a basin of attraction that L-BFGS can efficiently exploit. The corresponding training loss curves and prediction results are shown in Figure~\ref{fig:damped_wave_optimizer}.

\begin{figure}[htbp]
	\centering
	\includegraphics[width=0.5\textwidth]{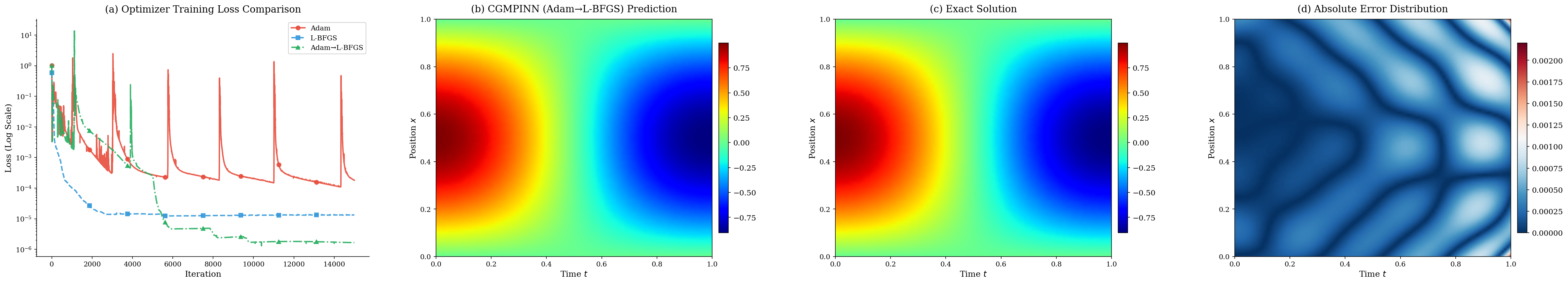}
	\caption{Optimizer comparison for 1D damped wave equation: (a) Training loss curves of different optimizers; (b) Prediction result of CGMPINN with Adam$\to$L-BFGS optimizer; (c) Exact solution of the equation; (d) Absolute error distribution of the prediction.}
	\label{fig:damped_wave_optimizer}
\end{figure}

Table~\ref{tab:1d_damped_wave_pinn_methods_performance} compares CGMPINN with five baseline methods on the damped wave equation. From the perspective of computational efficiency, CGMPINN is the most cost-effective method, requiring only 678.0\,s of CPU time, which is the lowest among all compared methods. This corresponds to reductions of 5.6\%, 41.8\%, and 77.2\% relative to the standard PINN, LNN-PINN, and gPINN, respectively. Despite this lower runtime, CGMPINN still achieves the best accuracy, attaining the lowest values of $e_2$, relative $e_2$, and $e_\infty$. In particular, compared with the standard PINN, its $e_2$ is reduced by 70.4\%; compared with the best-performing baseline methods in terms of $e_2$, namely lbPINN and LNN-PINN, the reductions are 27.5\% and 27.7\%, respectively. Its training loss is also very close to the minimum achieved by LNN-PINN. These results suggest that CGMPINN achieves a highly favorable balance between computational efficiency and predictive accuracy on this benchmark. The training loss curves, spatial-temporal prediction profiles, and pointwise error analysis are further illustrated in Figure~\ref{fig:damped_wave_methods}.

\begin{figure}[htbp]
	\centering
	\includegraphics[width=0.5\textwidth]{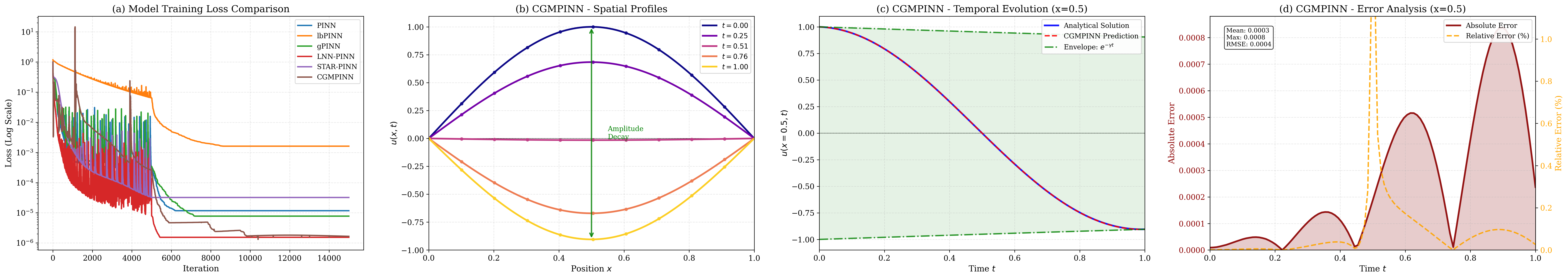}
	\caption{Performance comparison of different PINN variants for 1D damped wave equation: (a) Training loss curves of different models; (b) Spatial profiles of CGMPINN prediction at different time steps; (c) Temporal evolution of CGMPINN prediction at $x=0.5$ compared with the analytical solution; (d) Absolute and relative error of CGMPINN at $x=0.5$.}
	\label{fig:damped_wave_methods}
\end{figure}

\begin{table*}[t]
	\centering
	\caption{Performance Comparison of Different Optimizers for Solving Fisher-KPP Equation (with $D = 0.25$, $r = 4.0$)}
	\label{tab:fisher_kpp_optimizer_performance}
	\begin{tabular}{lccccc}
		\toprule
		Optimizer & $e_{\text{Loss}}$ & $e_2$ & $\text{Relative } e_2$ & $e_\infty$ & CPU (s) \\
		\midrule
		Adam               & $4.11e\text{-}5$ & $1.40e\text{-}3$ & $1.82e\text{-}3$ & $5.38e\text{-}3$ & $1319.9$   \\
		L-BFGS             & $9.33e\text{-}7$ & $1.62e\text{-}3$ & $2.11e\text{-}3$ & $5.87e\text{-}3$ & $1526.9$  \\
		Adam$\to$L-BFGS    & $\mathbf{3.62e\text{-}7}$ & $\mathbf{7.64e\text{-}4}$ & $\mathbf{9.94e\text{-}4}$ & $\mathbf{4.00e\text{-}3}$ & $1722.9$   \\
		\bottomrule
	\end{tabular}
\end{table*}

\begin{table*}[t]
	\centering
	\caption{Performance Comparison of Different Methods for Solving Fisher-KPP Equation (with $D = 0.25$, $r = 4.0$)}
	\label{tab:fisher_kpp_pinn_methods_performance}
	\begin{tabular}{lccccc}
		\toprule
		Method & $e_{\text{Loss}}$ & $e_2$ & $\text{Relative } e_2$ & $e_\infty$ & CPU (s) \\
		\midrule
		PINN               & $3.52e\text{-}6$ & $1.65e\text{-}3$ & $2.14e\text{-}3$ & $6.46e\text{-}3$ & $1527.0$   \\
		lbPINN             & $3.57e\text{-}4$ & $1.14e\text{-}3$ & $1.48e\text{-}3$ & $5.81e\text{-}3$ & $5711.2$   \\
		gPINN              & $3.45e\text{-}6$ & $1.72e\text{-}3$ & $2.24e\text{-}3$ & $7.03e\text{-}3$ & $7647.0$   \\
		LNN-PINN           & $1.62e\text{-}6$ & $4.66e\text{-}3$ & $6.06e\text{-}3$ & $2.14e\text{-}2$ & $2823.3$   \\
		STAR-PINN          & $7.23e\text{-}6$ & $1.10e\text{-}3$ & $1.43e\text{-}3$ & $7.38e\text{-}3$ & $1698.6$   \\
		CGMPINN            & $\mathbf{3.62e\text{-}7}$ & $\mathbf{7.64e\text{-}4}$ & $\mathbf{9.94e\text{-}4}$ & $\mathbf{4.00e\text{-}3}$ & $1722.9$   \\
		\bottomrule
	\end{tabular}
\end{table*}

\subsection{Advection-Diffusion Equation}
\label{subsec:advection_diffusion}

We consider a 1D advection-diffusion equation in the advection-dominated regime, a well-known challenging benchmark for PINNs due to the steep traveling wavefront that induces spectral bias and optimization instability. The governing equation with periodic boundary conditions is
\begin{equation}
	\begin{cases}
		u_t + a\, u_x - \nu u_{xx} = 0,\quad (x,t) \in \Omega \times [0,1], \\
		u(x, 0) = \sin(\pi x),\quad x \in \Omega, \\
		u(-1, t) = u(1, t),\;
		u_x(-1, t) = u_x(1, t),\quad t \in [0,1],
	\end{cases}
	\label{eq:advection_diffusion}
\end{equation}
with the exact solution
\begin{equation}
	u(x, t) = e^{-\nu \pi^2 t} \sin(\pi (x - a t)),
	\label{eq:advection_diffusion_solution}
\end{equation}
where $\Omega=[-1,1]$, $a = 1.0$ is the advection coefficient, and $\nu = 10^{-2}$ is the diffusion coefficient. All methods in this subsection employ a four-hidden-layer network with 50 neurons per layer, trained on 3{,}000 randomly sampled interior collocation points, 300 boundary points, and 300 initial condition points, and evaluated on 10{,}000 uniformly distributed test points.

The optimizer comparison in Table~\ref{tab:advection_diffusion_optimizer_performance} shows that the Adam$\to$L-BFGS strategy outperforms the alternatives, reducing $e_2$ by 51.2\% and $e_\infty$ by 46.3\% relative to pure L-BFGS, with a comparable training time. The performance gap between L-BFGS and Adam$\to$L-BFGS is narrower for this problem than for the previous benchmarks, suggesting that L-BFGS alone is relatively effective for the advection-diffusion dynamics, though Adam pre-training still provides a measurable benefit. The training loss curves and prediction results are shown in Figure~\ref{fig:advection_diffusion_optimizer}.

\begin{figure}[htbp]
	\centering
	\includegraphics[width=0.5\textwidth]{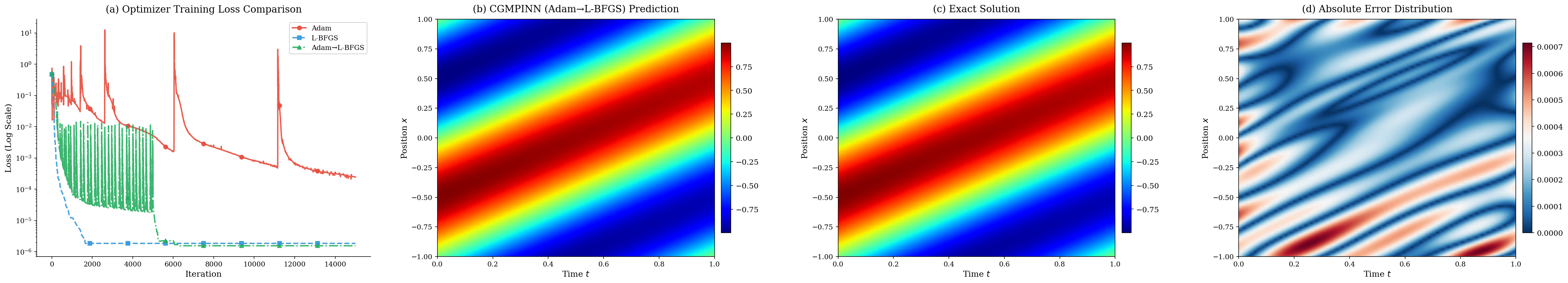}
	\caption{Optimizer comparison for 1D advection-diffusion equation: (a) Training loss curves of different optimizers; (b) Prediction result of CGMPINN with Adam$\to$L-BFGS optimizer; (c) Exact solution of the equation; (d) Absolute error distribution of the prediction.}
	\label{fig:advection_diffusion_optimizer}
\end{figure}

Table~\ref{tab:advection_diffusion_pinn_methods_performance} compares the performance of CGMPINN with five baseline methods on the advection-diffusion equation. CGMPINN delivers the best accuracy, achieving the lowest values of $e_2$, relative $e_2$, and $e_\infty$ among all methods. Specifically, its $e_2$ is $2.57\times10^{-4}$, corresponding to reductions of 64.0\%, 42.1\%, and 41.7\% relative to the standard PINN, lbPINN, and STAR-PINN, respectively. Its $e_\infty$ is also the smallest at $7.14\times10^{-4}$, indicating more uniformly accurate predictions across the domain and a 65.5\% reduction compared with the standard PINN. Moreover, CGMPINN attains the lowest training loss ($1.54\times10^{-6}$), suggesting effective optimization on this benchmark. Although its CPU time (532.4\,s) is slightly higher than that of the standard PINN (446.6\,s), it remains substantially lower than those of lbPINN (2018.2\,s) and gPINN (1885.6\,s), demonstrating that the accuracy gains are obtained with only modest additional computational cost. The training loss curves and error distributions are further illustrated in Figure~\ref{fig:advection_diffusion_methods}.

\begin{figure}[htbp]
	\centering
	\includegraphics[width=0.5\textwidth]{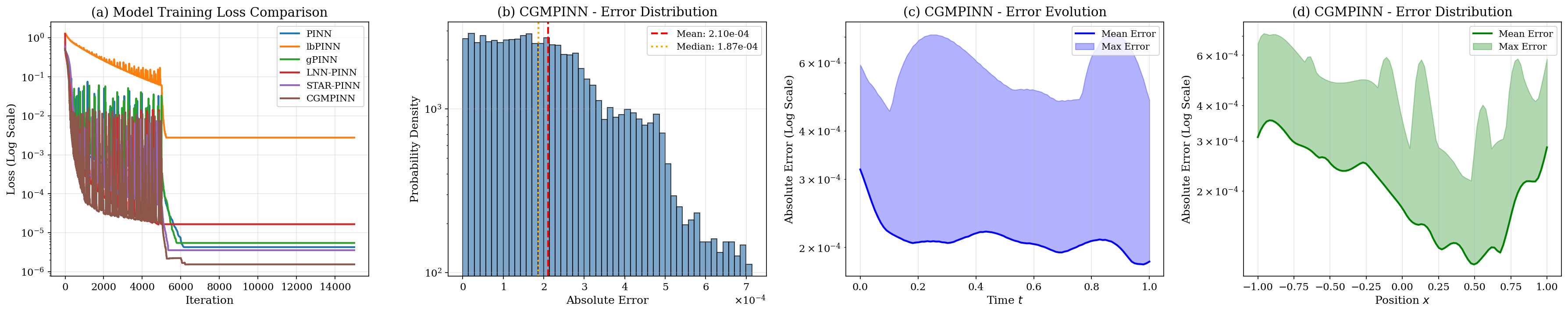}
	\caption{Performance comparison of different PINN variants for 1D advection-diffusion equation: (a) Training loss curves of different models; (b) Absolute error probability density distribution of CGMPINN; (c) Temporal evolution of the mean and max absolute error of CGMPINN; (d) Mean and max absolute error of CGMPINN in the x spatial direction.}
	\label{fig:advection_diffusion_methods}
\end{figure}

As illustrated in Figure~\ref{fig:advection_diffusion_term_balance},
the temporal derivative $u_t$ is predominantly balanced by the
advection term $-a u_x$ across the spatial domain at $t = 0.25$,
$0.5$, and $0.75$, while the diffusion contribution $\nu u_{xx}$
remains an order of magnitude smaller, confirming the
advection-dominated regime of this problem. The PDE residual is
consistently bounded at the $\mathcal{O}(10^{-3})$ level over all
examined time instants, indicating that the learned solution
satisfies the governing equation with high fidelity throughout the
temporal domain.

\begin{figure}[htbp]
	\centering
	\includegraphics[width=0.5\textwidth]{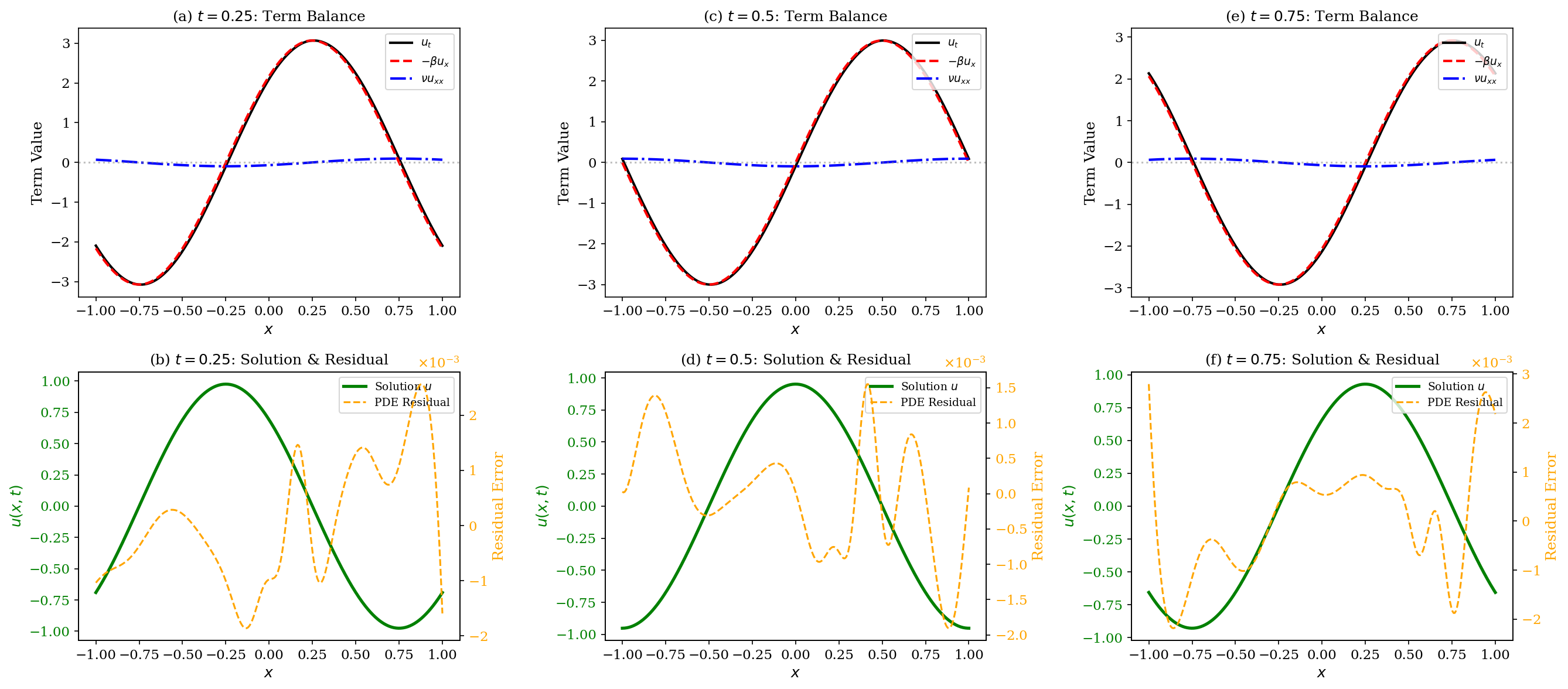}
	\caption{Term balance, solution and residual visualization for 1D advection-diffusion equation: (a)(c)(e) Balance of temporal derivative, advection and diffusion terms at $t=0.25$, $t=0.5$ and $t=0.75$; (b)(d)(f) Predicted solution and PDE residual distribution at the corresponding time steps.}
	\label{fig:advection_diffusion_term_balance}
\end{figure}

\subsection{Fisher-KPP Equation}
\label{subsec:fisher_kpp}

Finally, we evaluate CGMPINN on the Fisher-KPP equation, a nonlinear reaction-diffusion PDE that admits traveling wave solutions with steep fronts. The strong coupling between the nonlinear reaction and diffusion terms makes this problem a demanding test case. The governing equation is
\begin{equation}
	\begin{cases}
		u_t = D\, u_{xx} + r\, u(1 - u), \quad (x,t) \in \Omega \times [0,2], \\
		u(x,0) = u_0(x), \quad x \in \Omega, \\
		u(x,t) = g(x,t), \quad (x,t) \in \partial\Omega \times [0,2],
	\end{cases}
	\label{eq:fisher_kpp}
\end{equation}
with the Ablowitz--Zeppetella exact traveling wave solution
\begin{equation}
	u(x,t) = \frac{1}{\left(1 + \exp\left(\lambda (x - ct)\right)\right)^2},
	\label{eq:fisher_kpp_solution}
\end{equation}
where the waveform steepness parameter and wave speed are determined by
\begin{equation}
	\lambda = \sqrt{\frac{r}{6D}}, \quad c = 5\sqrt{\frac{Dr}{6}}.
	\label{eq:fisher_kpp_parameters}
\end{equation}
We set $\Omega=[-5,5]$, diffusion coefficient $D = 0.25$, and growth rate $r = 4.0$. Owing to the stronger nonlinearity of this equation compared with the previous cases, all methods in this subsection adopt a moderately enlarged network with four hidden layers of 80 neurons each, trained on 8{,}000 randomly sampled interior collocation points, 400 boundary points, and 400 initial condition points, and evaluated on 10{,}000 uniformly distributed test points.

Table~\ref{tab:fisher_kpp_optimizer_performance} presents the optimizer comparison. The Adam$\to$L-BFGS strategy yields the best performance, reducing $e_{\text{Loss}}$ by 61.2\% and $e_2$ by 52.8\% compared to pure L-BFGS. An interesting observation is that the pure L-BFGS optimizer produces a higher $e_2$ ($1.62\times10^{-3}$) than the Adam optimizer ($1.40\times10^{-3}$) despite achieving a lower training loss, suggesting that L-BFGS converges to a sharper minimum that does not generalize as well on this nonlinear problem. The two-stage strategy mitigates this issue by leveraging the broader exploration of Adam before the local refinement of L-BFGS. The training loss curves and prediction results are shown in Figure~\ref{fig:fisher_kpp_optimizer}.

\begin{figure}[htbp]
	\centering
	\includegraphics[width=0.5\textwidth]{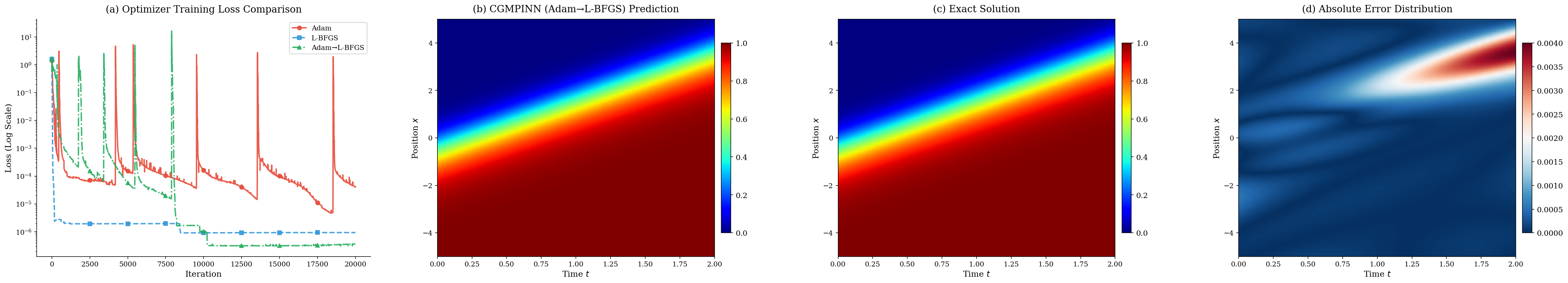}
	\caption{Optimizer comparison for 1D Fisher-KPP equation: (a) Training loss curves of different optimizers; (b) Prediction result of CGMPINN with Adam$\to$L-BFGS optimizer; (c) Exact solution of the equation; (d) Absolute error distribution of the prediction.}
	\label{fig:fisher_kpp_optimizer}
\end{figure}

Table~\ref{tab:fisher_kpp_pinn_methods_performance} summarizes the performance comparison among the considered methods. CGMPINN achieves the best results across all reported metrics, with $e_2 = 7.64\times10^{-4}$ and $e_\infty = 4.00\times10^{-3}$. Compared with the canonical PINN ($e_2 = 1.65\times10^{-3}$), CGMPINN reduces $e_2$ by 53.7\%. The second-best method in terms of $e_2$ is STAR-PINN ($1.10\times10^{-3}$), which is still outperformed by CGMPINN by 30.5\%. Notably, LNN-PINN ($e_2 = 4.66\times10^{-3}$) fails to accurately capture the traveling-wave solution for this nonlinear problem, yielding an error nearly three times larger than that of the canonical PINN despite attaining a lower training loss. This discrepancy between training loss and test accuracy reflects the challenge of optimizing PINNs for nonlinear PDEs with steep fronts. Among the remaining baselines, lbPINN and gPINN achieve only modest improvements in $e_2$ over the canonical PINN, at the expense of 3.7$\times$ and 5.0$\times$ higher training costs, respectively. By contrast, CGMPINN requires only 1722.9\,s, representing a 12.8\% increase over the standard PINN, while delivering the best overall accuracy and robustly capturing both the wave speed and the front profile. The training loss curves, wavefront tracking, and wave speed estimation results are further illustrated in Figure~\ref{fig:fisher_kpp_methods}.

\begin{figure}[htbp]
	\centering
	\includegraphics[width=0.5\textwidth]{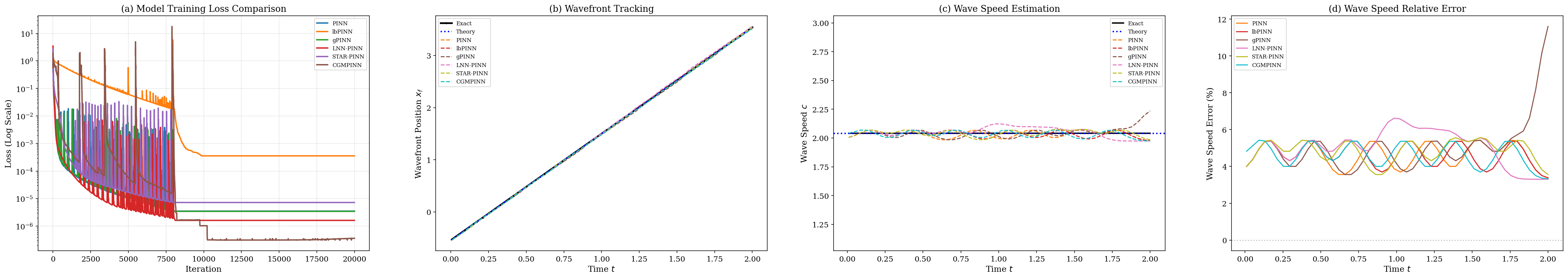}
	\caption{Performance comparison of different PINN variants for 1D Fisher-KPP equation: (a) Training loss curves of different models; (b) Wavefront tracking results of different PINN models; (c) Wave speed estimation results of different PINN models; (d) Temporal evolution of wave speed relative error for different models.}
	\label{fig:fisher_kpp_methods}
\end{figure}

As shown in Figure~\ref{fig:fisher_kpp_term_balance}, the temporal
derivative $u_t$ is nearly entirely balanced by the reaction term
$ru(1-u)$ across the spatial domain at $t = 0.5$, $1.0$, and $1.5$,
whereas the diffusion contribution $D u_{xx}$ is comparatively
negligible, confirming the reaction-dominated nature of the
Fisher-KPP dynamics during front propagation. The PDE residual
remains at the $\mathcal{O}(10^{-3})$--$\mathcal{O}(10^{-4})$ level
over all examined time instants and decreases monotonically as time
progresses, demonstrating that the predicted solution satisfies the
governing equation with high accuracy throughout the temporal
evolution.

\begin{figure}[htbp]
	\centering
	\includegraphics[width=0.5\textwidth]{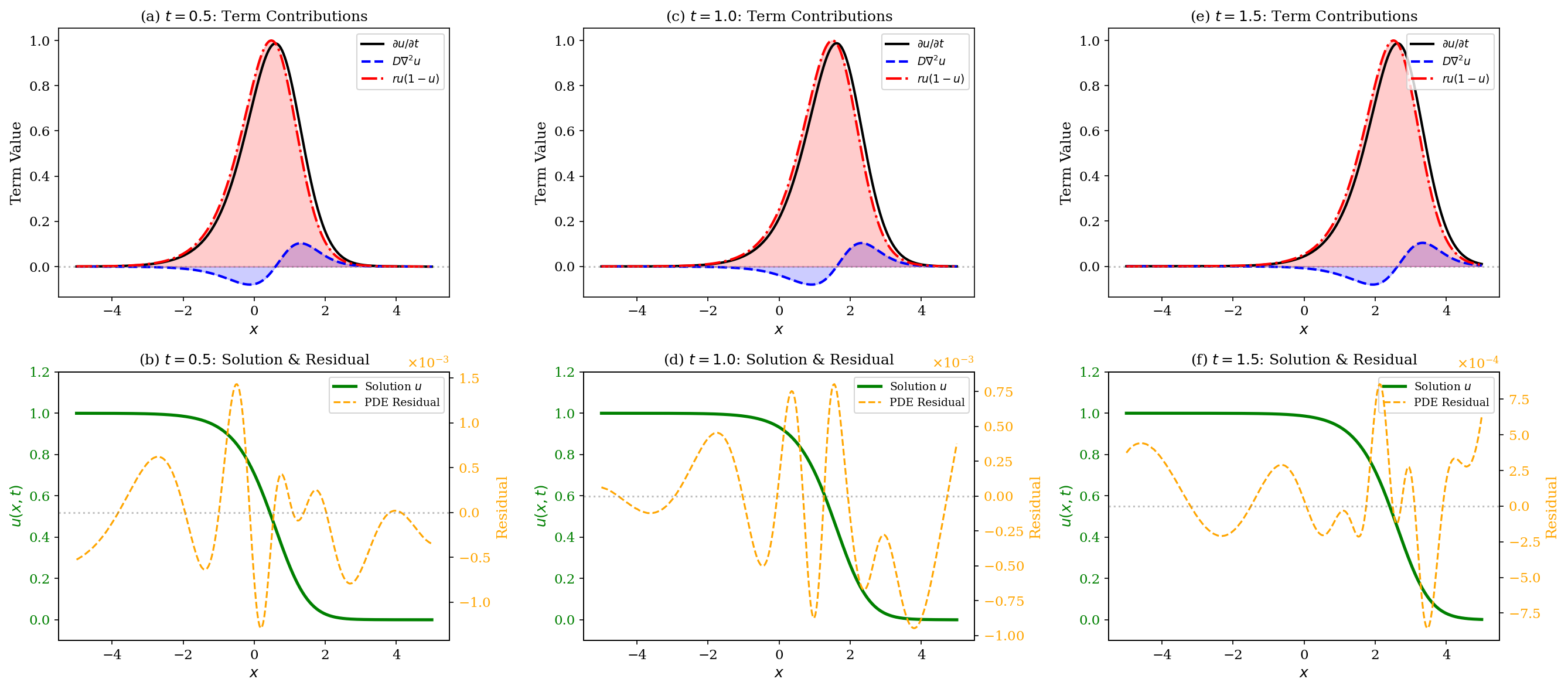}
	\caption{Term contribution, solution and residual visualization for 1D Fisher-KPP equation: (a)(c)(e) PDE term contributions at $t=0.5$, $t=1.0$ and $t=1.5$; (b)(d)(f) Predicted solution and PDE residual distribution at the corresponding time steps.}
	\label{fig:fisher_kpp_term_balance}
\end{figure}

	\section{Conclusion and Future Work}
	\label{sec:conclusion}
	
	In this paper, we proposed CGMPINN, a curriculum-guided physics-informed neural network framework that integrates Gaussian mixture modeling with dynamic curriculum learning for solving partial differential equations. The core idea is to employ a GMM to statistically characterize the distribution of PDE residuals, thereby providing a principled, data-driven quantification of local learning difficulty. A smooth curriculum schedule then progressively transitions training focus from easy, well-conditioned regions to hard, uncertain ones, while a precision-based variance modulation mechanism further prioritizes reliable clusters during early training. This dual curriculum is unified through a shared curriculum parameter and optionally coupled with ReLoBRaLo-based self-adaptive loss balancing. Systematic evaluations on six benchmark PDEs demonstrate that CGMPINN consistently achieves the lowest $e_2$ and $e_\infty$ errors among all compared methods, with reductions of up to 97.8\% in $e_2$ relative to the canonical PINN. Notably, these accuracy gains incur no significant computational overhead. On the theoretical front, we established formal guarantees for the proposed framework, including uniform equivalence between the curriculum-weighted and standard PDE losses, sublinear convergence of the gradient norm, and a generalization bound with an explicit weighting-induced bias characterization, thereby providing principled justification for the curriculum-guided reweighting strategy. An ablation study further confirms that the two core components---GMM-based difficulty quantification and curriculum scheduling---are complementary: neither component alone matches the performance of their combination, and omitting the curriculum schedule can lead to optimization failure on challenging problems.
	
	Despite the encouraging results, several directions merit further investigation. The current experiments are restricted to one- and two-dimensional domains; extending CGMPINN to high-dimensional PDEs and assessing the scalability of GMM-based difficulty quantification in such settings is an important next step. Developing principled strategies for automatic selection of the number of mixture components $K$ and the update frequency $k_{\mathrm{upd}}$, for instance via information-theoretic criteria, would further reduce manual tuning. Combining the curriculum-guided mechanism with domain decomposition approaches and extending the framework to inverse problems with sparse, noisy observations are also natural generalizations. On the theoretical side, the present analysis establishes convergence, loss equivalence, and generalization guarantees under idealized full-batch gradient dynamics; extending these results to a complete iterate-level convergence theory for the practical Adam$\to$L-BFGS training scheme and establishing tighter bounds that formally certify the advantage of dynamic curriculum reweighting over standard PINNs remain valuable open directions.

	\bibliographystyle{unsrt}
	\bibliography{ref}
	
	
\clearpage
\appendix
\numberwithin{equation}{section}
\allowdisplaybreaks[2]

\section{Theoretical Analysis and Proofs for CGMPINN}
\label{app:theory}

We analyze the optimization geometry induced by CGMPINN and a
generalization bound for the weighted PDE objective when the weighting rule
is fixed after a refresh step.
Unless stated otherwise, notation follows Section~\ref{sec:cgmpinn}; the
GMM--curriculum construction is as in
Eqs.~\eqref{eq:gmm_posterior}--\eqref{eq:weighted_pde_loss} and the total
loss is Eq.~\eqref{eq:total_loss}.

Convergence statements below concern \emph{idealized} full-batch gradient
descent on the time-varying total loss $\mathcal{L}_{\mathrm{total}}(\cdot;k)$,
which isolates GMM-based reweighting and objective drift.
The hybrid Adam$\to$L-BFGS procedure in experiments targets the same
objective but is not analyzed iterate-wise here; what we establish are
structural properties of the objective and of fixed-rule statistical risk
(Remark~\ref{rem:theory_limitations}).

\subsection{Auxiliary Assumptions}
\label{app:assumptions}

We first state the assumptions used throughout the analysis.

\begin{assumption}[Smoothness of the time-varying total loss]
	\label{ass:smooth_total}
	For every iteration $k$, the function
	$\mathcal{L}_{\mathrm{total}}(\cdot;k):\mathbb{R}^p\to\mathbb{R}$
	(where $p$ is the parameter dimension introduced in
	Section~\ref{sec:cgmpinn})
	is continuously differentiable and $L$-smooth, namely
	\begin{equation}
		\begin{split}
			\bigl\|
			\nabla_{\theta}\mathcal{L}_{\mathrm{total}}(\theta_1;k)
			-
			\nabla_{\theta}\mathcal{L}_{\mathrm{total}}(\theta_2;k)
			\bigr\|
			&\leqslant
			L\|\theta_1-\theta_2\|, \\
			&\quad \forall \theta_1,\theta_2\in\mathbb{R}^p.
		\end{split}
		\label{eq:app_smoothness}
	\end{equation}
	Moreover, $\mathcal{L}_{\mathrm{total}}(\theta;k)$ admits a finite
	uniform lower bound: there exists $\mathcal{L}_{\inf}\in\mathbb{R}$
	such that
	\begin{equation}
		\mathcal{L}_{\mathrm{total}}(\theta;k)\geqslant \mathcal{L}_{\inf},
		\quad
		\forall \theta\in\mathbb{R}^p,\, k\geqslant 0.
		\label{eq:app_lower_bound}
	\end{equation}
\end{assumption}

\begin{remark}[Discussion of Assumption~\ref{ass:smooth_total}]
	\label{rem:disc_smooth}
	Both conditions are mild and standard in first-order optimization
	analysis of neural networks.
	The $L$-smoothness holds globally for networks with smooth activations
	(e.g., $\tanh$, sigmoid, GELU, Swish) on any bounded parameter region,
	and almost everywhere for ReLU-type networks; in either case, the
	gradient remains Lipschitz on the compact iterate set encountered
	during training.
	The uniform lower bound is automatic for the composite MSE objective,
	since $\mathcal{L}_{\mathrm{total}}(\theta;k)\geqslant 0$
	for every $\theta$ and $k$, so we may take
	$\mathcal{L}_{\inf}=0$.
\end{remark}

\begin{assumption}[Bounded GMM variances]
	\label{ass:bounded_variances}
	There exist constants
	$0<\underline{\sigma}^{2}\leqslant \overline{\sigma}^{2}<\infty$
	such that, at every GMM refresh step and for every Gaussian component
	$m$,
	\begin{equation}
		\underline{\sigma}^{2}
		\leqslant
		\sigma_m^{2}
		\leqslant
		\overline{\sigma}^{2}.
		\label{eq:app_variance_bound}
	\end{equation}
\end{assumption}

\begin{remark}[Discussion of Assumption~\ref{ass:bounded_variances}]
	\label{rem:disc_variances}
	The two-sided bound on $\sigma_m^2$ is non-restrictive in practice.
	The upper bound $\overline{\sigma}^2$ follows from the boundedness of
	PDE residuals on the compact training domain.
	The strict positive lower bound $\underline{\sigma}^2>0$ is routinely
	enforced in EM-based GMM solvers via covariance regularization
	(e.g., the \texttt{reg\_covar} parameter of
	\texttt{scikit-learn}'s \texttt{GaussianMixture}), which adds a small
	floor to each component variance and rules out degenerate clusters.
\end{remark}

\begin{assumption}[Summable drift of the time-varying objective]
	\label{ass:summable_drift}
	Consider the full-batch gradient descent iteration
	\begin{equation}
		\theta_{k+1}
		=
		\theta_k
		-
		\eta \nabla_{\theta}\mathcal{L}_{\mathrm{total}}(\theta_k;k),
		\label{eq:app_gd_update}
	\end{equation}
	with the step size $\eta>0$ as introduced in
	Theorem~\ref{thm:main_nonconvex}.
	Assume that there exists a nonnegative sequence $\{\xi_k\}_{k\geqslant 0}$
	such that
	\begin{equation}
		\sum_{k=0}^{\infty}\xi_k<\infty
		\label{eq:app_summable_xi}
	\end{equation}
	and
	\begin{equation}
		\mathcal{L}_{\mathrm{total}}(\theta_{k+1};k+1)
		-
		\mathcal{L}_{\mathrm{total}}(\theta_{k+1};k)
		\leqslant
		\xi_k,
		\quad \forall k.
		\label{eq:app_drift_condition}
	\end{equation}
\end{assumption}

\begin{remark}[Discussion of Assumption~\ref{ass:summable_drift}]
	\label{rem:disc_drift}
	The objective drift in CGMPINN originates from three coupled sources:
	the curriculum parameter $\tau(k)$, the periodic GMM re-fitting that
	updates the sample weights $\{w_i\}$, and the ReLoBRaLo adaptive loss
	coefficients $\{\lambda_c(k)\}$.
	Summability of $\xi_k$ is consistent with the design of CGMPINN.
	By Eq.~\eqref{eq:tau}, the curriculum parameter saturates at
	$\tau(k)\equiv 1$ once $k\geqslant k_{\max}\,c_{\mathrm{sat}}$, after
	which the curriculum-induced drift vanishes identically.
	The GMM is refreshed only every $k_{\mathrm{upd}}$ iterations, and
	both its parameters $\Theta_{\mathrm{GMM}}$ and the resulting weights
	$\{w_i\}$ stabilize as the residual distribution settles; an analogous
	stabilization occurs for the ReLoBRaLo weights through the EMA in
	Eq.~\eqref{eq:relobralo_ema}.
	Together, these mechanisms ensure that $\xi_k$ decays sufficiently
	fast to be summable, which validates the use of
	Assumption~\ref{ass:summable_drift} for the convergence analysis below.
\end{remark}

\subsection{Convergence Analysis}
\label{app:convergence}

\subsubsection{Equivalence between weighted and standard PDE losses}

\begin{theorem}[Uniform equivalence between weighted and standard PDE losses]
	\label{thm:weight_equiv}
	Under Assumption~\ref{ass:bounded_variances}, define
	\begin{equation}
		\underline{v}
		:=
		\frac{\underline{\sigma}^{2}+\epsilon}
		{\overline{\sigma}^{2}+\epsilon}
		\in(0,1].
		\label{eq:app_v_under}
	\end{equation}
	Then, for every Gaussian component $m$ and every curriculum parameter
	$\tau\in[0,1]$, the component-level weight
	\begin{equation}
		w_m^{\mathrm{comp}}(\tau)
		=
		w_m^{\mathrm{CL}}(\tau)\,\tilde v_m(\tau)
		\label{eq:app_component_weight}
	\end{equation}
	satisfies
	\begin{equation}
		e^{-\beta}\underline{v}
		\leqslant
		w_m^{\mathrm{comp}}(\tau)
		\leqslant
		1.
		\label{eq:app_component_weight_bound}
	\end{equation}
	Define the pre-normalized sample weight by
	\begin{equation}
		w_i^{\mathrm{raw}}(\tau)
		:=
		\sum_{m=1}^{K}\gamma_{im}\,w_m^{\mathrm{comp}}(\tau).
		\label{eq:app_raw_weight}
	\end{equation}
	Then
	\begin{equation}
		e^{-\beta}\underline{v}
		\leqslant
		w_i^{\mathrm{raw}}(\tau)
		\leqslant
		1.
		\label{eq:app_raw_weight_bound}
	\end{equation}
	Let
	\begin{equation}
		\begin{aligned}
			c_- &:= \frac{N_{\Omega}e^{-\beta}\underline{v}}{N_{\Omega}+\epsilon},\\
			c_+ &:= \frac{N_{\Omega}}{N_{\Omega}e^{-\beta}\underline{v}+\epsilon}.
		\end{aligned}
		\label{eq:app_weight_constants}
	\end{equation}
	After the normalization in Eq.~\eqref{eq:weight_normalize}, the final
	sample weight $w_i$ satisfies
	\begin{equation}
		c_- \leqslant w_i \leqslant c_+.
		\label{eq:app_final_weight_bound}
	\end{equation}
	Consequently, for every $\theta$,
	\begin{equation}
		c_-\,
		\mathcal{L}_{\mathrm{PDE}}(\theta)
		\leqslant
		\mathcal{L}_{\mathrm{PDE}}^{w}(\theta)
		\leqslant
		c_+\,
		\mathcal{L}_{\mathrm{PDE}}(\theta).
		\label{eq:app_loss_equivalence}
	\end{equation}
\end{theorem}

\begin{proof}
	By Eq.~\eqref{eq:easy_hard_weights},
	\[
	w_m^{\mathrm{easy}}=\exp(-\beta\tilde d_m),
	\quad
	w_m^{\mathrm{hard}}=\exp(-\beta(1-\tilde d_m)).
	\]
	Since $\tilde d_m\in[0,1]$, one has
	\begin{equation}
		e^{-\beta}\leqslant w_m^{\mathrm{easy}}\leqslant 1,
		\quad
		e^{-\beta}\leqslant w_m^{\mathrm{hard}}\leqslant 1.
		\label{eq:app_easy_hard_bound}
	\end{equation}
	By Eq.~\eqref{eq:curriculum_weight},
	\[
	w_m^{\mathrm{CL}}(\tau)
	=
	(1-\tau)w_m^{\mathrm{easy}}
	+
	\tau w_m^{\mathrm{hard}},
	\]
	which is a convex combination of
	$w_m^{\mathrm{easy}}$ and $w_m^{\mathrm{hard}}$.
	Hence,
	\begin{equation}
		e^{-\beta}\leqslant w_m^{\mathrm{CL}}(\tau)\leqslant 1.
		\label{eq:app_cl_bound}
	\end{equation}
	
	Next, by Assumption~\ref{ass:bounded_variances},
	\begin{equation}
		\frac{1}{\overline{\sigma}^{2}+\epsilon}
		\leqslant
		(\sigma_m^{2}+\epsilon)^{-1}
		\leqslant
		\frac{1}{\underline{\sigma}^{2}+\epsilon}.
		\label{eq:app_inverse_variance_bound}
	\end{equation}
	Using Eq.~\eqref{eq:precision_factor}, we have
	\begin{align}
		v_m
		&=
		\frac{(\sigma_m^{2}+\epsilon)^{-1}}
		{\max_{1\leqslant l\leqslant K}(\sigma_l^{2}+\epsilon)^{-1}}
		\geqslant
		\frac{(\overline{\sigma}^{2}+\epsilon)^{-1}}
		{(\underline{\sigma}^{2}+\epsilon)^{-1}}
		=
		\frac{\underline{\sigma}^{2}+\epsilon}
		{\overline{\sigma}^{2}+\epsilon}
		=
		\underline{v}.
		\label{eq:app_vm_lower}
	\end{align}
	Trivially, $v_m\leqslant 1$. Then by Eq.~\eqref{eq:effective_precision},
	\begin{equation}
		\underline{v}\leqslant \tilde v_m(\tau)\leqslant 1.
		\label{eq:app_vtilde_bound}
	\end{equation}
	Combining \eqref{eq:app_cl_bound},
	\eqref{eq:app_vtilde_bound}, and
	Eq.~\eqref{eq:combined_weight}, we obtain
	\[
	e^{-\beta}\underline{v}
	\leqslant
	w_m^{\mathrm{comp}}(\tau)
	=
	w_m^{\mathrm{CL}}(\tau)\tilde v_m(\tau)
	\leqslant
	1,
	\]
	which proves \eqref{eq:app_component_weight_bound}.
	
	Now consider the pre-normalized sample weight
	\[
	w_i^{\mathrm{raw}}(\tau)
	=
	\sum_{m=1}^{K}\gamma_{im}\,w_m^{\mathrm{comp}}(\tau).
	\]
	Since $\gamma_{im}\geqslant 0$ and
	$\sum_{m=1}^{K}\gamma_{im}=1$ by Eq.~\eqref{eq:gmm_posterior},
	$w_i^{\mathrm{raw}}(\tau)$ is a convex combination of
	$\{w_m^{\mathrm{comp}}(\tau)\}_{m=1}^{K}$.
	Therefore,
	\[
	e^{-\beta}\underline{v}
	\leqslant
	w_i^{\mathrm{raw}}(\tau)
	\leqslant
	1,
	\]
	which proves \eqref{eq:app_raw_weight_bound}.
	
	By Eq.~\eqref{eq:weight_normalize}, the final sample weight is
	\begin{equation}
		w_i
		=
		\frac{N_{\Omega}\,w_i^{\mathrm{raw}}(\tau)}
		{\sum_{j=1}^{N_{\Omega}}w_j^{\mathrm{raw}}(\tau)+\epsilon}.
		\label{eq:app_exact_normalization}
	\end{equation}
	Using \eqref{eq:app_raw_weight_bound}, we have
	\[
	\sum_{j=1}^{N_{\Omega}}w_j^{\mathrm{raw}}(\tau)
	\leqslant
	N_{\Omega},
	\quad
	\sum_{j=1}^{N_{\Omega}}w_j^{\mathrm{raw}}(\tau)
	\geqslant
	N_{\Omega}e^{-\beta}\underline{v}.
	\]
	Substituting these bounds into \eqref{eq:app_exact_normalization} yields
	\[
	w_i
	\geqslant
	\frac{N_{\Omega}e^{-\beta}\underline{v}}{N_{\Omega}+\epsilon}
	=
	c_-,
	\quad
	w_i
	\leqslant
	\frac{N_{\Omega}}{N_{\Omega}e^{-\beta}\underline{v}+\epsilon}
	=
	c_+,
	\]
	which proves \eqref{eq:app_final_weight_bound}.
	
	Finally, since
	$r_{\mathrm{PDE},i}^{2}\geqslant 0$ for all $i$,
	\begin{align}
		\mathcal{L}_{\mathrm{PDE}}^{w}(\theta)
		&=
		\frac{1}{N_{\Omega}}
		\sum_{i=1}^{N_{\Omega}}
		w_i\,r_{\mathrm{PDE},i}^{2}
		\nonumber\\
		&\geqslant
		\frac{1}{N_{\Omega}}
		\sum_{i=1}^{N_{\Omega}}
		c_-\,r_{\mathrm{PDE},i}^{2}
		\nonumber\\
		&=
		c_-\mathcal{L}_{\mathrm{PDE}}(\theta),
		\label{eq:app_loss_lower}
	\end{align}
	and similarly,
	\begin{align}
		\mathcal{L}_{\mathrm{PDE}}^{w}(\theta)
		\leqslant
		c_+\mathcal{L}_{\mathrm{PDE}}(\theta).
		\label{eq:app_loss_upper}
	\end{align}
	Combining \eqref{eq:app_loss_lower} and \eqref{eq:app_loss_upper}
	proves \eqref{eq:app_loss_equivalence}.
\end{proof}

\begin{corollary}[Consistency with the standard PDE objective]
	\label{cor:loss_consistency}
	Under the assumptions of Theorem~\ref{thm:weight_equiv}, if a sequence
	$\{\theta_n\}$ satisfies
	\begin{equation}
		\mathcal{L}_{\mathrm{PDE}}^{w}(\theta_n)\to 0,
		\label{eq:app_weighted_loss_zero}
	\end{equation}
	then necessarily
	\begin{equation}
		\mathcal{L}_{\mathrm{PDE}}(\theta_n)\to 0.
		\label{eq:app_unweighted_loss_zero}
	\end{equation}
\end{corollary}

\begin{proof}
	By \eqref{eq:app_loss_equivalence},
	\begin{equation}
		0
		\leqslant
		\mathcal{L}_{\mathrm{PDE}}(\theta_n)
		\leqslant
		c_-^{-1}\mathcal{L}_{\mathrm{PDE}}^{w}(\theta_n)\to 0.
		\label{eq:app_loss_consistency_proof}
	\end{equation}
\end{proof}

\subsubsection{Sublinear convergence for the time-varying total loss}

\begin{theorem}[Sublinear convergence of the time-varying total loss]
	\label{thm:nonconvex_convergence}
	Suppose Assumptions~\ref{ass:smooth_total} and
	\ref{ass:summable_drift} hold, and choose the step size
	$\eta\in(0,1/L]$.
	Then the iterates generated by \eqref{eq:app_gd_update} satisfy
	\begin{equation}
		\sum_{k=0}^{\infty}
		\bigl\|
		\nabla_{\theta}\mathcal{L}_{\mathrm{total}}(\theta_k;k)
		\bigr\|^2
		<\infty,
		\label{eq:app_grad_square_sum}
	\end{equation}
	where $\Xi:=\sum_{k=0}^{\infty}\xi_k$.
	Moreover, for every integer $K\geqslant 0$,
	\begin{equation}
		\begin{split}
			\min_{0\leqslant t\leqslant K}
			\bigl\|
			\nabla_{\theta}\mathcal{L}_{\mathrm{total}}(\theta_t;t)
			\bigr\|^2
			&\leqslant
			\frac{2}{\eta(K+1)} \\
			&\quad \times
			\Bigl(
			\mathcal{L}_{\mathrm{total}}(\theta_0;0)-\mathcal{L}_{\inf}+\Xi
			\Bigr).
		\end{split}
		\label{eq:app_sublinear_rate}
	\end{equation}
	Consequently,
	\begin{equation}
		\begin{aligned}
			\min_{0\leqslant t\leqslant K}
			\bigl\|
			\nabla_{\theta}\mathcal{L}_{\mathrm{total}}(\theta_t;t)
			\bigr\|^2
			& =
			O(K^{-1}),\\
			& \bigl\|
			\nabla_{\theta}\mathcal{L}_{\mathrm{total}}(\theta_k;k)
			\bigr\|\to 0.
			\label{eq:app_grad_vanish}
		\end{aligned}
	\end{equation}
\end{theorem}

\begin{proof}
	Since $\mathcal{L}_{\mathrm{total}}(\cdot;k)$ is $L$-smooth for each
	fixed $k$, the descent lemma yields
	\begin{align}
		\mathcal{L}_{\mathrm{total}}(\theta_{k+1};k)
		&\leqslant
		\mathcal{L}_{\mathrm{total}}(\theta_k;k)
		+ 
		\frac{L}{2}\|\theta_{k+1}-\theta_k\|^2 \nonumber\\
		& \quad
		+ \Bigl\langle
		\nabla_{\theta}\mathcal{L}_{\mathrm{total}}(\theta_k;k),
		\theta_{k+1}-\theta_k
		\Bigr\rangle
		\label{eq:app_descent_lemma}
	\end{align}
	Substituting the update rule \eqref{eq:app_gd_update}, we obtain
	\begin{align}
		\mathcal{L}_{\mathrm{total}}(\theta_{k+1};k)
		&\leqslant
		\mathcal{L}_{\mathrm{total}}(\theta_k;k)
		-
		\eta
		\bigl\|
		\nabla_{\theta}\mathcal{L}_{\mathrm{total}}(\theta_k;k)
		\bigr\|^2
		\nonumber\\
		&\quad
		+
		\frac{L\eta^2}{2}
		\bigl\|
		\nabla_{\theta}\mathcal{L}_{\mathrm{total}}(\theta_k;k)
		\bigr\|^2.
		\label{eq:app_after_substitution}
	\end{align}
	Since $\eta\leqslant 1/L$, one has
	$1-\frac{L\eta}{2}\geqslant \frac12$, and therefore
	\begin{equation}
		\mathcal{L}_{\mathrm{total}}(\theta_{k+1};k)
		\leqslant
		\mathcal{L}_{\mathrm{total}}(\theta_k;k)
		-
		\frac{\eta}{2}
		\bigl\|
		\nabla_{\theta}\mathcal{L}_{\mathrm{total}}(\theta_k;k)
		\bigr\|^2.
		\label{eq:app_basic_descent}
	\end{equation}
	
	By Assumption~\ref{ass:summable_drift},
	\begin{equation}
		\mathcal{L}_{\mathrm{total}}(\theta_{k+1};k+1)
		-
		\mathcal{L}_{\mathrm{total}}(\theta_{k+1};k)
		\leqslant
		\xi_k.
		\label{eq:app_drift_again}
	\end{equation}
	Combining \eqref{eq:app_basic_descent} and
	\eqref{eq:app_drift_again}, we obtain
	\begin{equation}
		\mathcal{L}_{\mathrm{total}}(\theta_{k+1};k+1)
		\leqslant
		\mathcal{L}_{\mathrm{total}}(\theta_k;k)
		-
		\frac{\eta}{2}
		\bigl\|
		\nabla_{\theta}\mathcal{L}_{\mathrm{total}}(\theta_k;k)
		\bigr\|^2
		+
		\xi_k.
		\label{eq:app_descent_with_drift}
	\end{equation}
	
	Summing \eqref{eq:app_descent_with_drift} from $k=0$ to $K$ yields
	\begin{equation}
		\begin{split}
			\mathcal{L}_{\mathrm{total}}(\theta_{K+1};K+1)
			&\leqslant
			\mathcal{L}_{\mathrm{total}}(\theta_0;0)
			+
			\sum_{k=0}^{K}\xi_k \\
			&\quad
			-
			\frac{\eta}{2}
			\sum_{k=0}^{K}
			\bigl\|
			\nabla_{\theta}\mathcal{L}_{\mathrm{total}}(\theta_k;k)
			\bigr\|^2.
		\end{split}
		\label{eq:app_telescoping}
	\end{equation}
	Using the lower bound
	$\mathcal{L}_{\mathrm{total}}(\theta_{K+1};K+1)\geqslant \mathcal{L}_{\inf}$,
	we obtain
	\begin{equation}
		\begin{split}
			\frac{\eta}{2}
			\sum_{k=0}^{K}
			\bigl\|
			\nabla_{\theta}\mathcal{L}_{\mathrm{total}}(\theta_k;k)
			\bigr\|^2
			&\leqslant
			\mathcal{L}_{\mathrm{total}}(\theta_0;0)-\mathcal{L}_{\inf}
			+\sum_{k=0}^{K}\xi_k \\
			&\leqslant
			\mathcal{L}_{\mathrm{total}}(\theta_0;0)-\mathcal{L}_{\inf}+\Xi.
		\end{split}
		\label{eq:app_sum_bound}
	\end{equation}
	Letting $K\to\infty$ proves \eqref{eq:app_grad_square_sum}.
	
	Moreover,
	\begin{equation}
		(K+1)\min_{0\leqslant t\leqslant K}
		\bigl\|
		\nabla_{\theta}\mathcal{L}_{\mathrm{total}}(\theta_t;t)
		\bigr\|^2
		\leqslant
		\sum_{k=0}^{K}
		\bigl\|
		\nabla_{\theta}\mathcal{L}_{\mathrm{total}}(\theta_k;k)
		\bigr\|^2.
		\label{eq:app_min_bound}
	\end{equation}
	Combining \eqref{eq:app_sum_bound} and \eqref{eq:app_min_bound} gives
	\eqref{eq:app_sublinear_rate}.
	
	Finally, since the nonnegative series in
	\eqref{eq:app_grad_square_sum} converges, its general term must
	converge to zero, which proves \eqref{eq:app_grad_vanish}.
\end{proof}

\begin{corollary}[Stationarity of accumulation points]
	\label{cor:stationary_points}
	Assume, in addition to
	Theorem~\ref{thm:nonconvex_convergence}, that the sequence
	$\{\theta_k\}$ is bounded and that there exists a continuously
	differentiable limit objective $\mathcal{L}_{\infty}(\theta)$ such that
	\begin{equation}
		\begin{aligned}
			\sup_{\theta\in \mathcal K}
			\bigl\|
			\nabla_{\theta}\mathcal{L}_{\mathrm{total}}(\theta;k)
			-
			\nabla_{\theta}\mathcal{L}_{\infty}(\theta)
			\bigr\|
			\to 0, \\
			\text{for every compact } \mathcal K\subset\mathbb{R}^p.
		\end{aligned}
		\label{eq:app_uniform_gradient_convergence}
	\end{equation}
	Then every accumulation point of $\{\theta_k\}$ is a stationary point of
	$\mathcal{L}_{\infty}(\theta)$.
\end{corollary}

\begin{proof}
	Let $\bar{\theta}$ be an accumulation point of $\{\theta_k\}$.
	Then there exists a subsequence $\{\theta_{k_j}\}$ such that
	$\theta_{k_j}\to\bar{\theta}$.
	Since $\{\theta_k\}$ is bounded, all $\theta_{k_j}$ and $\bar{\theta}$
	belong to a common compact set $\mathcal K$.
	We write
	\begin{align}
		\bigl\|
		\nabla_{\theta}\mathcal{L}_{\infty}(\bar{\theta})
		\bigr\|
		&\leqslant
		\bigl\|
		\nabla_{\theta}\mathcal{L}_{\infty}(\bar{\theta})
		-
		\nabla_{\theta}\mathcal{L}_{\mathrm{total}}(\bar{\theta};k_j)
		\bigr\|
		\nonumber\\
		&\quad
		+
		\bigl\|
		\nabla_{\theta}\mathcal{L}_{\mathrm{total}}(\bar{\theta};k_j)
		-
		\nabla_{\theta}\mathcal{L}_{\mathrm{total}}(\theta_{k_j};k_j)
		\bigr\|
		\nonumber\\
		&\quad
		+
		\bigl\|
		\nabla_{\theta}\mathcal{L}_{\mathrm{total}}(\theta_{k_j};k_j)
		\bigr\|.
		\label{eq:app_stationary_split}
	\end{align}
	The first term tends to zero by
	\eqref{eq:app_uniform_gradient_convergence}.
	The second term tends to zero by Assumption~\ref{ass:smooth_total}:
	\begin{equation}
		\bigl\|
		\nabla_{\theta}\mathcal{L}_{\mathrm{total}}(\bar{\theta};k_j)
		-
		\nabla_{\theta}\mathcal{L}_{\mathrm{total}}(\theta_{k_j};k_j)
		\bigr\|
		\leqslant
		L\|\bar{\theta}-\theta_{k_j}\|
		\to 0.
		\label{eq:app_stationary_smoothness}
	\end{equation}
	The third term tends to zero by \eqref{eq:app_grad_vanish}.
	Hence,
	\begin{equation}
		\nabla_{\theta}\mathcal{L}_{\infty}(\bar{\theta})=0.
		\label{eq:app_stationary_limit}
	\end{equation}
\end{proof}

\subsubsection{PL-based convergence with drift control}

\begin{assumption}[Uniform Polyak--\L{}ojasiewicz condition with a common reference level]
	\label{ass:PL}
	There exist constants $\mu>0$ and
	$\mathcal{L}^{\star}\in\mathbb{R}$ such that, for every $k$ and every
	$\theta$,
	\begin{equation}
		\mathcal{L}_{\mathrm{total}}(\theta;k)\geqslant \mathcal{L}^{\star},
		\label{eq:app_PL_lower}
	\end{equation}
	and
	\begin{equation}
		\frac{1}{2}
		\bigl\|
		\nabla_{\theta}\mathcal{L}_{\mathrm{total}}(\theta;k)
		\bigr\|^2
		\geqslant
		\mu
		\bigl(
		\mathcal{L}_{\mathrm{total}}(\theta;k)-\mathcal{L}^{\star}
		\bigr).
		\label{eq:app_PL_condition}
	\end{equation}
\end{assumption}

\begin{assumption}[Uniformly bounded drift amplitude]
	\label{ass:bounded_drift}
	In addition to Assumption~\ref{ass:summable_drift}, assume that there
	exists $\bar\xi\geqslant 0$ such that
	\begin{equation}
		0\leqslant \xi_k\leqslant \bar\xi,
		\quad \forall k.
		\label{eq:app_xi_upper}
	\end{equation}
\end{assumption}

\begin{theorem}[PL-based convergence to a drift-controlled neighborhood]
	\label{thm:PL_neighborhood}
	Suppose Assumptions~\ref{ass:smooth_total}, \ref{ass:PL}, and
	\ref{ass:bounded_drift} hold, and choose $\eta\in(0,1/L]$.
	Define the objective gap
	\begin{equation}
		E_k
		:=
		\mathcal{L}_{\mathrm{total}}(\theta_k;k)-\mathcal{L}^{\star}.
		\label{eq:app_gap_def}
	\end{equation}
	Then
	\begin{equation}
		E_{k+1}
		\leqslant
		(1-\eta\mu)E_k+\xi_k.
		\label{eq:app_PL_recursion}
	\end{equation}
	Consequently, for every $k\geqslant 0$,
	\begin{equation}
		E_k
		\leqslant
		(1-\eta\mu)^kE_0
		+
		\sum_{t=0}^{k-1}(1-\eta\mu)^{k-1-t}\xi_t.
		\label{eq:app_PL_unrolled}
	\end{equation}
	In particular, since $\xi_t\leqslant \bar\xi$,
	\begin{equation}
		\begin{aligned}
			E_k
			\leqslant\;&
			(1-\eta\mu)^kE_0
			+
			\frac{\bar\xi}{\eta\mu}
			\Bigl(
			1-(1-\eta\mu)^k
			\Bigr) \\
			\leqslant\;&
			(1-\eta\mu)^kE_0+\frac{\bar\xi}{\eta\mu}.
		\end{aligned}
		\label{eq:app_PL_neighborhood}
	\end{equation}
	Hence,
	\begin{equation}
		\limsup_{k\to\infty}E_k
		\leqslant
		\frac{\bar\xi}{\eta\mu}.
		\label{eq:app_PL_limsup}
	\end{equation}
	That is, the homogeneous term decays linearly, and the total loss is
	controlled by a drift-dominated neighborhood whose radius is
	proportional to $\bar\xi/(\eta\mu)$.
\end{theorem}

\begin{proof}
	From \eqref{eq:app_descent_with_drift}, we already have
	\begin{equation}
		\begin{split}
			\mathcal{L}_{\mathrm{total}}(\theta_{k+1};k+1)
			&\leqslant
			\mathcal{L}_{\mathrm{total}}(\theta_k;k)
			+
			\xi_k \\
			&\quad
			-
			\frac{\eta}{2}
			\bigl\|
			\nabla_{\theta}\mathcal{L}_{\mathrm{total}}(\theta_k;k)
			\bigr\|^2.
		\end{split}
		\label{eq:app_PL_start}
	\end{equation}
	By Assumption~\ref{ass:PL},
	\begin{equation}
		\frac{1}{2}
		\bigl\|
		\nabla_{\theta}\mathcal{L}_{\mathrm{total}}(\theta_k;k)
		\bigr\|^2
		\geqslant
		\mu
		\bigl(
		\mathcal{L}_{\mathrm{total}}(\theta_k;k)-\mathcal{L}^{\star}
		\bigr)
		=
		\mu E_k.
		\label{eq:app_PL_apply}
	\end{equation}
	Substituting \eqref{eq:app_PL_apply} into
	\eqref{eq:app_PL_start} yields
	\begin{equation}
		\mathcal{L}_{\mathrm{total}}(\theta_{k+1};k+1)
		\leqslant
		\mathcal{L}_{\mathrm{total}}(\theta_k;k)
		-
		\eta\mu E_k
		+
		\xi_k.
		\label{eq:app_PL_mid}
	\end{equation}
	Subtracting $\mathcal{L}^{\star}$ from both sides gives
	\eqref{eq:app_PL_recursion}.
	
	Let $q:=1-\eta\mu$.
	By induction on $k$, \eqref{eq:app_PL_recursion} implies
	\eqref{eq:app_PL_unrolled}.
	Using $\xi_t\leqslant \bar\xi$, we further obtain
	\begin{equation}
		\sum_{t=0}^{k-1}q^{k-1-t}\xi_t
		\leqslant
		\bar\xi\sum_{j=0}^{k-1}q^j
		=
		\bar\xi\,
		\frac{1-q^k}{1-q}
		=
		\bar\xi\,
		\frac{1-(1-\eta\mu)^k}{\eta\mu},
		\label{eq:app_geometric_series}
	\end{equation}
	which yields \eqref{eq:app_PL_neighborhood}. Taking the limit superior
	proves \eqref{eq:app_PL_limsup}.
\end{proof}

\begin{corollary}[Exact convergence under summable drift]
	\label{cor:PL_exact_conv}
	Under the assumptions of Theorem~\ref{thm:PL_neighborhood}, if in
	addition Assumption~\ref{ass:summable_drift} holds, namely
	$\sum_{k=0}^{\infty}\xi_k<\infty$, then
	\begin{equation}
		E_k\to 0,
		\qquad\text{i.e.,}\qquad
		\mathcal{L}_{\mathrm{total}}(\theta_k;k)\to \mathcal{L}^{\star}.
		\label{eq:app_PL_exact_conv}
	\end{equation}
\end{corollary}

\begin{proof}
	From \eqref{eq:app_PL_unrolled}, it suffices to show that
	\[
	\sum_{t=0}^{k-1}(1-\eta\mu)^{k-1-t}\xi_t\to 0.
	\]
	Let $q:=1-\eta\mu\in[0,1)$.
	Since $\sum_{t=0}^{\infty}\xi_t<\infty$, for any $\varepsilon>0$
	(local proof variable, distinct from the global numerical constant
	$\epsilon$) there exists a threshold index $T_0$ (not to be confused
	with the final time $T$ in Section~\ref{sec:cgmpinn}) such that
	\[
	\sum_{t=T_0}^{\infty}\xi_t<\frac{\varepsilon}{2}.
	\]
	For every $k>T_0$,
	\begin{align}
		\sum_{t=0}^{k-1}q^{k-1-t}\xi_t
		&=
		\sum_{t=0}^{T_0-1}q^{k-1-t}\xi_t
		+
		\sum_{t=T_0}^{k-1}q^{k-1-t}\xi_t
		\nonumber\\
		&\leqslant
		q^{k-T_0}\sum_{t=0}^{T_0-1}\xi_t
		+
		\sum_{t=T_0}^{k-1}\xi_t.
		\label{eq:app_tail_split}
	\end{align}
	The second term is smaller than $\varepsilon/2$ by construction.
	The first term tends to zero as $k\to\infty$ because $q\in[0,1)$.
	Hence, for all sufficiently large $k$,
	\[
	\sum_{t=0}^{k-1}q^{k-1-t}\xi_t<\varepsilon.
	\]
	Therefore,
	\[
	\sum_{t=0}^{k-1}q^{k-1-t}\xi_t\to 0.
	\]
	Together with \eqref{eq:app_PL_unrolled} and $q^kE_0\to 0$, this proves
	\eqref{eq:app_PL_exact_conv}.
\end{proof}

\begin{corollary}[Exact linear convergence in the zero-drift case]
	\label{cor:zero_drift_linear}
	Under the assumptions of Theorem~\ref{thm:PL_neighborhood}, if
	$\xi_k\equiv 0$, then
	\begin{equation}
		E_k\leqslant (1-\eta\mu)^kE_0,
		\quad \forall k\geqslant 0.
		\label{eq:app_zero_drift_linear}
	\end{equation}
	Hence,
	\begin{equation}
		\mathcal{L}_{\mathrm{total}}(\theta_k;k)\to \mathcal{L}^{\star}
		\quad\text{at a linear rate.}
		\label{eq:app_exact_linear}
	\end{equation}
\end{corollary}

\begin{proof}
	Setting $\xi_k\equiv 0$ in \eqref{eq:app_PL_unrolled} immediately gives
	\eqref{eq:app_zero_drift_linear}, from which
	\eqref{eq:app_exact_linear} follows.
\end{proof}

\subsection{Generalization Analysis}
\label{app:generalization}

To distinguish population risks from empirical losses, we define
\begin{equation}
	\mathcal{E}_{\mathrm{PDE}}(\theta)
	:=
	\mathbb{E}_{(\boldsymbol{x},t)\sim \rho}
	\bigl[
	r_{\mathrm{PDE}}^2(\boldsymbol{x},t;\theta)
	\bigr],
	\label{eq:app_population_unweighted}
\end{equation}
and
\begin{equation}
	\mathcal{E}_{\mathrm{PDE}}^{w}(\theta)
	:=
	\mathbb{E}_{(\boldsymbol{x},t)\sim \rho}
	\bigl[
	w(\boldsymbol{x},t)\,
	r_{\mathrm{PDE}}^2(\boldsymbol{x},t;\theta)
	\bigr],
	\label{eq:app_population_weighted}
\end{equation}
where $w(\boldsymbol{x},t)$ denotes the pointwise weight induced by the
same GMM-curriculum mechanism.

\begin{assumption}[Bounded residual class]
	\label{ass:bounded_residual}
	There exists a constant $M>0$ such that, for all $\theta$ and all
	$(\boldsymbol{x},t)\in\Omega\times[0,T]$,
	\begin{equation}
		|r_{\mathrm{PDE}}(\boldsymbol{x},t;\theta)|\leqslant M.
		\label{eq:app_residual_bound}
	\end{equation}
	Equivalently,
	\begin{equation}
		0\leqslant
		r_{\mathrm{PDE}}^2(\boldsymbol{x},t;\theta)
		\leqslant
		B:=M^2.
		\label{eq:app_squared_residual_bound}
	\end{equation}
\end{assumption}

\begin{assumption}[Bounded weight function]
	\label{ass:bounded_weight_function}
	There exists a constant $W\geqslant 1$ such that
	\begin{equation}
		0\leqslant w(\boldsymbol{x},t)\leqslant W,
		\quad
		\forall (\boldsymbol{x},t)\in\Omega\times[0,T].
		\label{eq:app_weight_function_bound}
	\end{equation}
\end{assumption}

\paragraph{Split-sample proof device.}
We decouple weight construction from the empirical process by fitting the
GMM on a reference sample and forming $w(\boldsymbol{x},t)$ from it; a
separate sample defines the weighted empirical loss
\begin{equation}
	\mathcal{L}_{\mathrm{PDE}}^{w}(\theta)
	=
	\frac{1}{N_{\Omega}}
	\sum_{i=1}^{N_{\Omega}}
	w(\boldsymbol{x}_i,t_i)\,
	r_{\mathrm{PDE},i}^{2}.
	\label{eq:app_empirical_weighted_loss}
\end{equation}
Conditioned on the reference draw, $w$ is deterministic (whereas the
implementation may reuse the same collocation points).
The bound below is for a fixed refresh/frozen rule and does not address
fully coupled adaptive reweighting across all iterations.

Let $\Theta\subseteq\mathbb{R}^p$ denote the parameter space over which
the network is optimized, and define the squared-residual function class
\begin{equation}
	\mathcal{F}
	:=
	\left\{
	(\boldsymbol{x},t)\mapsto
	r_{\mathrm{PDE}}^2(\boldsymbol{x},t;\theta)
	:\ \theta\in\Theta
	\right\}.
	\label{eq:app_function_class}
\end{equation}

\begin{theorem}[Generalization bound for the weighted empirical PDE loss under a frozen weighting rule]
	\label{thm:weighted_generalization}
	Under Assumptions~\ref{ass:bounded_residual} and
	\ref{ass:bounded_weight_function}, conditioned on the reference sample
	used to construct $w(\boldsymbol{x},t)$, with probability at least
	$1-\delta$ over the independent training sample
	$\{(\boldsymbol{x}_i,t_i)\}_{i=1}^{N_{\Omega}}$, the following holds
	simultaneously for all $\theta\in\Theta$:
	\begin{equation}
		\begin{aligned}
			\mathcal{E}_{\mathrm{PDE}}^{w}(\theta)
			\leqslant\;&
			\mathcal{L}_{\mathrm{PDE}}^{w}(\theta)
			+
			2W\,\mathfrak{R}_{N_{\Omega}}(\mathcal{F}) \\
			&+
			3BW\sqrt{\frac{\log(2/\delta)}{2N_{\Omega}}}.
		\end{aligned}
		\label{eq:app_weighted_generalization}
	\end{equation}
	where $\mathfrak{R}_{N_{\Omega}}(\mathcal{F})$ denotes the Rademacher
	complexity of $\mathcal{F}$.
\end{theorem}

\begin{proof}
	Define the weighted function class
	\begin{equation}
		w\mathcal{F}
		:=
		\left\{
		(\boldsymbol{x},t)\mapsto
		w(\boldsymbol{x},t)\,
		r_{\mathrm{PDE}}^2(\boldsymbol{x},t;\theta)
		:\ \theta\in\Theta
		\right\}.
		\label{eq:app_weighted_function_class}
	\end{equation}
	Since
	$0\leqslant r_{\mathrm{PDE}}^2(\boldsymbol{x},t;\theta)\leqslant B$
	and
	$0\leqslant w(\boldsymbol{x},t)\leqslant W$,
	every function in $w\mathcal{F}$ is bounded in $[0,BW]$.
	
	By the standard Rademacher-complexity-based uniform convergence bound,
	with probability at least $1-\delta$,
	\begin{equation}
		\begin{split}
			\sup_{\theta\in\Theta}
			\Bigl(
			\mathcal{E}_{\mathrm{PDE}}^{w}(\theta)
			-
			\mathcal{L}_{\mathrm{PDE}}^{w}(\theta)
			\Bigr)
			&\leqslant
			2\,\mathfrak{R}_{N_{\Omega}}(w\mathcal{F}) \\
			&\quad
			+
			3BW\sqrt{\frac{\log(2/\delta)}{2N_{\Omega}}}.
		\end{split}
		\label{eq:app_standard_rademacher}
	\end{equation}
	It therefore suffices to estimate
	$\mathfrak{R}_{N_{\Omega}}(w\mathcal{F})$.
	For a fixed training sample,
	\begin{align}
		\widehat{\mathfrak{R}}(w\mathcal{F})
		&=
		\mathbb{E}_{\sigma}
		\left[
		\sup_{\theta\in\Theta}
		\frac{1}{N_{\Omega}}
		\sum_{i=1}^{N_{\Omega}}
		\sigma_i
		w(\boldsymbol{x}_i,t_i)
		r_{\mathrm{PDE},i}^{2}
		\right],
		\label{eq:app_empirical_rademacher}
	\end{align}
	where $\{\sigma_i\}_{i=1}^{N_{\Omega}}$ are i.i.d.\ Rademacher random
	variables (not to be confused with the GMM component variances
	$\sigma_m^2$ in Section~\ref{sec:cgmpinn}).
	For each $i$, the map
	$\phi_i(u)=w(\boldsymbol{x}_i,t_i)u$
	satisfies $\phi_i(0)=0$ and is $W$-Lipschitz.
	Hence, by the contraction inequality,
	\begin{equation}
		\widehat{\mathfrak{R}}(w\mathcal{F})
		\leqslant
		W\,\widehat{\mathfrak{R}}(\mathcal{F}).
		\label{eq:app_contraction}
	\end{equation}
	Taking expectation over the training sample yields
	\begin{equation}
		\mathfrak{R}_{N_{\Omega}}(w\mathcal{F})
		\leqslant
		W\,\mathfrak{R}_{N_{\Omega}}(\mathcal{F}).
		\label{eq:app_rademacher_weighted}
	\end{equation}
	Substituting \eqref{eq:app_rademacher_weighted} into
	\eqref{eq:app_standard_rademacher} proves
	\eqref{eq:app_weighted_generalization}.
\end{proof}

\begin{corollary}[Generalization bound for the original PDE risk under a frozen weight rule]
	\label{cor:unweighted_generalization}
	Under the assumptions of
	Theorem~\ref{thm:weighted_generalization}, with probability at least
	$1-\delta$, uniformly for all $\theta\in\Theta$,
	\begin{equation}
		\begin{aligned}
			\mathcal{E}_{\mathrm{PDE}}(\theta)
			\leqslant\;&
			\mathcal{L}_{\mathrm{PDE}}^{w}(\theta)
			+
			2W\,\mathfrak{R}_{N_{\Omega}}(\mathcal{F}) \\
			&+
			3BW\sqrt{\frac{\log(2/\delta)}{2N_{\Omega}}}
			+
			B\|w-1\|_{L^1(\rho)}.
		\end{aligned}
		\label{eq:app_unweighted_generalization}
	\end{equation}
\end{corollary}

\begin{proof}
	By definition,
	\begin{equation}
		\mathcal{E}_{\mathrm{PDE}}(\theta)
		-
		\mathcal{E}_{\mathrm{PDE}}^{w}(\theta)
		=
		\mathbb{E}_{(\boldsymbol{x},t)\sim \rho}
		\Bigl[
		(1-w(\boldsymbol{x},t))
		r_{\mathrm{PDE}}^2(\boldsymbol{x},t;\theta)
		\Bigr].
		\label{eq:app_bias_decomposition}
	\end{equation}
	Taking absolute values and using
	\eqref{eq:app_squared_residual_bound}, we obtain
	\begin{equation}
		\begin{aligned}
			\bigl|
			\mathcal{E}_{\mathrm{PDE}}(\theta)
			-
			\mathcal{E}_{\mathrm{PDE}}^{w}(\theta)
			\bigr|
			&\leqslant
			B\,
			\mathbb{E}_{(\boldsymbol{x},t)\sim \rho}
			|1-w(\boldsymbol{x},t)| \\
			&=
			B\|w-1\|_{L^1(\rho)}.
		\end{aligned}
		\label{eq:app_bias_bound}
	\end{equation}
	Therefore,
	\begin{equation}
		\mathcal{E}_{\mathrm{PDE}}(\theta)
		\leqslant
		\mathcal{E}_{\mathrm{PDE}}^{w}(\theta)
		+
		B\|w-1\|_{L^1(\rho)}.
		\label{eq:app_transfer_to_unweighted}
	\end{equation}
	Applying Theorem~\ref{thm:weighted_generalization} to
	$\mathcal{E}_{\mathrm{PDE}}^{w}(\theta)$ yields
	\eqref{eq:app_unweighted_generalization}.
\end{proof}

\begin{corollary}[Pseudo-dimension form]
	\label{cor:pdim_bound}
	Assume, in addition, that the pseudo-dimension of $\mathcal{F}$ is
	$d_{\mathcal{F}}$ (not to be confused with the spatial dimension $d$
	of $\Omega$ in Section~\ref{sec:cgmpinn}).
	Then there exists a universal constant $C_{\mathrm{R}}>0$
	(distinct from the convergence constant $C$ in
	Theorem~\ref{thm:main_nonconvex}) such that, with probability
	at least $1-\delta$,
	\begin{equation}
		\begin{aligned}
			\mathcal{E}_{\mathrm{PDE}}(\theta)
			\leqslant\;&
			\mathcal{L}_{\mathrm{PDE}}^{w}(\theta)
			+
			C_{\mathrm{R}}BW
			\sqrt{
				\frac{
					d_{\mathcal{F}}\log(eN_{\Omega}/d_{\mathcal{F}})+\log(1/\delta)
				}{
					N_{\Omega}
				}
			} \\
			&+
			B\|w-1\|_{L^1(\rho)},
			\quad
			\forall \theta\in\Theta.
		\end{aligned}
		\label{eq:app_pdim_bound}
	\end{equation}
\end{corollary}

\begin{proof}
	For uniformly bounded function classes with pseudo-dimension
	$d_{\mathcal{F}}$, the standard Rademacher complexity estimate yields
	\begin{equation}
		\mathfrak{R}_{N_{\Omega}}(\mathcal{F})
		\leqslant
		C_{\mathrm{R}}B
		\sqrt{
			\frac{
				d_{\mathcal{F}}\log(eN_{\Omega}/d_{\mathcal{F}})
			}{
				N_{\Omega}
			}
		}.
		\label{eq:app_pdim_rademacher}
	\end{equation}
	Substituting \eqref{eq:app_pdim_rademacher} into
	\eqref{eq:app_unweighted_generalization} proves the claim.
\end{proof}

\begin{remark}[Interpretation of the frozen-rule generalization result]
	\label{rem:frozen_rule_scope}
	Theorem~\ref{thm:weighted_generalization} characterizes the statistical
	behavior of the induced weighted empirical PDE objective at a fixed
	refresh step after the weighting rule has been frozen.
	It therefore provides a clean generalization guarantee for the induced
	weighted objective, but it should not be interpreted as a complete
	generalization theory for the fully coupled adaptive reweighting process
	across all training iterations.
\end{remark}

\begin{remark}[Interpretation of the bias term]
	\label{rem:bias_term}
	The additional term
	$B\|w-1\|_{L^1(\rho)}$
	in \eqref{eq:app_unweighted_generalization} quantifies the bias
	incurred by optimizing the weighted PDE risk instead of the original
	unweighted one.
	Under the current CGMPINN design, this term does not automatically
	vanish as $\tau(k)\to 1$, because the terminal hard-curriculum weights
	remain difficulty-dependent.
	If asymptotically unbiased weighting is desired, an additional design
	constraint must be imposed, for example by explicitly forcing the
	terminal component weights to approach $1$.
\end{remark}

\begin{remark}[What the appendix establishes and what it does not]
	\label{rem:theory_limitations}
	The results support three claims:
	(i)~uniform equivalence of the curriculum-weighted and standard empirical
	PDE losses;
	(ii)~stationarity-type guarantees for idealized full-batch descent on the
	time-varying total loss under summable objective drift; and
	(iii)~a standard uniform convergence bound for the weighted empirical PDE
	loss at a fixed refresh step (frozen weights), plus an explicit bias term
	toward the unweighted population risk.
	They do not furnish iterate-level convergence for Adam$\to$L-BFGS nor prove
	that CGMPINN strictly improves generalization over standard PINNs; a
	rigorous analysis of the hybrid optimizer under adaptive reweighting is
	left to future work.
\end{remark}

\begin{table*}[t]
	\centering
	\caption{Ablation study comparing GMMPINN, CLPINN, and CGMPINN on all benchmark PDEs. All experimental settings are identical to those in Section~\ref{sec:numerical_experiments}. Bold indicates the best result in each group.}
	\label{tab:ablation_all}
	\begin{tabular}{llccccc}
			\toprule
			Problem & Method & $e_{\text{Loss}}$ & $e_2$ & $\text{Relative } e_2$ & $e_\infty$ & CPU (s) \\
			\midrule
			\multirow{3}{*}{1D Poisson}
			& GMMPINN & $3.68e\text{+}3$ & $2.17e\text{+}0$ & $2.01e\text{+}0$ & $3.53e\text{+}0$ & $821.4$ \\
			& CLPINN  & $2.34e\text{-}3$ & $3.29e\text{-}3$ & $3.04e\text{-}3$ & $4.22e\text{-}3$ & $818.9$ \\
			& CGMPINN & $\mathbf{7.40e\text{-}4}$ & $\mathbf{1.96e\text{-}4}$ & $\mathbf{1.81e\text{-}4}$ & $\mathbf{3.73e\text{-}4}$ & $895.9$ \\
			\midrule
			\multirow{3}{*}{2D Poisson}
			& GMMPINN & $2.97e\text{-}4$ & $5.66e\text{-}4$ & $7.10e\text{-}4$ & $3.39e\text{-}3$ & $1442.0$ \\
			& CLPINN  & $7.59e\text{-}5$ & $5.39e\text{-}4$ & $6.76e\text{-}4$ & $4.80e\text{-}3$ & $1575.8$ \\
			& CGMPINN & $\mathbf{6.65e\text{-}5}$ & $\mathbf{4.65e\text{-}4}$ & $\mathbf{5.83e\text{-}4}$ & $\mathbf{2.95e\text{-}3}$ & $1453.2$ \\
			\midrule
			\multirow{3}{*}{Heat}
			& GMMPINN & $2.76e\text{-}2$ & $1.79e\text{-}3$ & $1.57e\text{-}3$ & $2.29e\text{-}2$ & $1010.7$ \\
			& CLPINN  & $9.90e\text{-}5$ & $1.78e\text{-}3$ & $1.56e\text{-}3$ & $2.51e\text{-}2$ & $1216.6$ \\
			& CGMPINN & $\mathbf{1.28e\text{-}5}$ & $\mathbf{4.04e\text{-}4}$ & $\mathbf{3.54e\text{-}4}$ & $\mathbf{2.86e\text{-}3}$ & $1429.96$ \\
			\midrule
			\multirow{3}{*}{Damped Wave}
			& GMMPINN & $2.11e\text{-}6$ & $4.36e\text{-}4$ & $9.15e\text{-}4$ & $3.00e\text{-}3$ & $617.0$ \\
			& CLPINN  & $2.92e\text{-}6$ & $5.10e\text{-}4$ & $1.07e\text{-}3$ & $3.92e\text{-}3$ & $581.0$ \\
			& CGMPINN & $\mathbf{1.56e\text{-}6}$ & $\mathbf{3.37e\text{-}4}$ & $\mathbf{7.07e\text{-}4}$ & $\mathbf{2.20e\text{-}3}$ & $678.0$ \\
			\midrule
			\multirow{3}{*}{Advection-Diffusion}
			& GMMPINN & $3.18e\text{-}6$ & $6.87e\text{-}4$ & $1.02e\text{-}3$ & $1.43e\text{-}3$ & $517.9$ \\
			& CLPINN  & $\mathbf{1.41e\text{-}6}$ & $3.37e\text{-}4$ & $5.00e\text{-}4$ & $9.71e\text{-}4$ & $517.4$ \\
			& CGMPINN & $1.66e\text{-}6$ & $\mathbf{2.57e\text{-}4}$ & $\mathbf{3.82e\text{-}4}$ & $\mathbf{7.14e\text{-}4}$ & $532.4$ \\
			\midrule
			\multirow{3}{*}{Fisher-KPP}
			& GMMPINN & $1.54e\text{-}6$ & $3.06e\text{-}3$ & $3.97e\text{-}3$ & $1.18e\text{-}2$ & $1632.9$ \\
			& CLPINN  & $1.00e\text{-}6$ & $1.25e\text{-}3$ & $1.63e\text{-}3$ & $4.51e\text{-}3$ & $1602.3$ \\
			& CGMPINN & $\mathbf{3.12e\text{-}7}$ & $\mathbf{7.64e\text{-}4}$ & $\mathbf{9.94e\text{-}4}$ & $\mathbf{4.00e\text{-}3}$ & $1722.9$ \\
			\bottomrule
	\end{tabular}
\end{table*}


\section{Ablation Study}
\label{app:ablation}

To disentangle the contributions of the two core components in CGMPINN---the GMM-based difficulty quantification and the curriculum learning schedule---we conduct an ablation study across all benchmark PDEs. Three variants are compared:

\begin{itemize}
	\item \textbf{GMMPINN} (GMM-based weighting only): A GMM is fitted to the PDE residuals and component-level difficulty scores are computed, but no curriculum schedule is applied. Instead, the component weights are set to $\exp(\beta\tilde{d}_m)$, which statically up-weights harder components throughout training. The precision-based variance factor, when enabled, is applied without $\tau$-modulation.
	
	\item \textbf{CLPINN} (curriculum learning only): The $\tau$-controlled easy-to-hard curriculum schedule is applied, but difficulty is measured per-sample from the raw squared residual magnitude rather than from GMM component-level statistics. No GMM fitting or precision modulation is used.
	
	\item \textbf{CGMPINN} (full framework): The complete method combining GMM-based component-level difficulty quantification, $\tau$-controlled curriculum scheduling, and precision-based variance modulation, as described in Section~\ref{sec:cgmpinn}.
\end{itemize}

\noindent All problem configurations, network architectures, training hyperparameters, and evaluation metrics are identical to those in Section~\ref{sec:numerical_experiments}. The Adam$\to$L-BFGS two-stage optimization strategy is used for all three variants.

Table~\ref{tab:ablation_all} reports the ablation results. CGMPINN consistently achieves the lowest $e_2$, relative $e_2$, and $e_\infty$ across all benchmarks, confirming that neither component alone matches the performance of their combination.

\textbf{1D Poisson.}\;
GMMPINN fails catastrophically ($e_2=2.17$), demonstrating that statically up-weighting high-residual regions without a curriculum schedule overwhelms the optimizer early in training and leads to divergence. CLPINN achieves reasonable accuracy ($e_2=3.29\times10^{-3}$), but CGMPINN further reduces the error by over an order of magnitude ($e_2=1.96\times10^{-4}$), indicating that the GMM-based structural information significantly enhances the per-sample curriculum when properly integrated.

\textbf{2D Poisson.}\;
All three variants achieve comparable accuracy, with CGMPINN providing modest improvements in both $e_2$ and $e_\infty$. The relatively uniform difficulty landscape of this problem reduces the benefit of structured difficulty quantification, though the GMM-curriculum combination still yields the best overall performance.

\textbf{Heat equation.}\;
GMMPINN and CLPINN yield similar $e_2$ values ($\approx1.8\times10^{-3}$), but CGMPINN reduces the error by approximately $4\times$ to $4.04\times10^{-4}$. This suggests that the multi-scale temporal dynamics in the heat equation benefit substantially from the combination of component-level difficulty quantification and progressive curriculum scheduling.

\textbf{Damped wave equation.}\;
CGMPINN achieves the lowest errors among the three variants, reducing $e_2$ by 23\% relative to GMMPINN and 34\% relative to CLPINN. Notably, CLPINN performs slightly worse than GMMPINN on this problem, suggesting that for oscillatory solutions the GMM-based clustering captures difficulty structure more effectively than per-sample residual magnitude alone.

\textbf{Advection-diffusion equation.}\;
CLPINN already provides a substantial improvement over GMMPINN ($e_2$ reduced by 51\%), reflecting the importance of progressive curriculum scheduling for advection-dominated problems where sharp fronts emerge. CGMPINN further reduces the error by 24\% relative to CLPINN, confirming the added value of the GMM-based difficulty quantification.

\textbf{Fisher-KPP equation.}\;
Both components contribute meaningfully: CLPINN reduces $e_2$ by 59\% relative to GMMPINN, and CGMPINN achieves an additional 39\% reduction over CLPINN. The nonlinear traveling-wave dynamics of this problem benefit from both the structured difficulty identification and the progressive training schedule.

In summary, the ablation study confirms that the two components of CGMPINN---GMM-based difficulty quantification and curriculum scheduling---are complementary. The GMM module provides structured, component-level difficulty information that the per-sample curriculum alone cannot capture, while the curriculum schedule prevents the premature focus on hard regions that causes GMMPINN to fail on challenging problems. Their combination consistently yields the best performance, with the most pronounced advantages on problems featuring strong nonlinearity, multi-scale dynamics, or sharp gradients.

\end{document}